\pdfoutput=1
\documentclass[3p]{elsarticle}

\usepackage{caption}
\usepackage{subcaption}
\usepackage{amssymb,amsmath,amsthm}
\usepackage{bm}
\usepackage{siunitx}
\usepackage{tikz}
\usepackage{hyperref}

\usepackage{array}

\newtheorem{lemma}{Lemma}

\journal{Journal of Computational Physics}

\biboptions{numbers,sort&compress}
    
\begin{document}

\begin{frontmatter}

\title{Theory and Application of Shapelets to the Analysis of Surface Self-assembly Imaging}
    
\author[cs]{Robert Suderman}

\author[cs]{Daniel Lizotte}
\ead{dlizotte@uwaterloo.ca}

\author[che,win]{Nasser Mohieddin Abukhdeir\corref{corr}}
\cortext[corr]{Corresponding author}
\ead{nmabukhdeir@uwaterloo.ca}

\address[cs]{Department of Computer Science, University of Waterloo, Waterloo, Ontario N2L 3G1, Canada}
\address[che]{Department of  Chemical Engineering, University of Waterloo, Waterloo, Ontario N2L 3G1, Canada}
\address[win]{Waterloo Institute for Nanotechnology, University of Waterloo, Waterloo, Ontario N2L 3G1, Canada}

\begin{abstract}
A method for quantitative analysis of local pattern strength and defects in surface self-assembly imaging is presented and applied to images of stripe and hexagonal ordered domains.
The presented method uses ``shapelet'' functions which were originally developed for quantitative analysis of images of galaxies ($\propto \SI{e20}{\meter}$).
In this work, they are used instead to quantify the presence of translational order in surface self-assembled films ($\propto \SI{e-9}{\meter}$) through reformulation into ``steerable'' filters.
The resulting method is both computationally efficient (with respect to the number of filter evaluations), robust to variation in pattern feature shape, and, unlike previous approaches, is applicable to a wide variety of pattern types.
An application of the method is presented which uses a nearest-neighbour analysis to distinguish between uniform (defect-free) and non-uniform (strained, defect-containing) regions within imaged self-assembled domains, both with striped and hexagonal patterns.
\end{abstract}


\end{frontmatter}

\newcommand{\ie}{\textit{i.e.{}}}

\section{Introduction}

Modern microscopy techniques are producing an ever-increasing amount of high-resolution imaging data of self-assembled materials.  
There are thousands of images of such films in the self-assembly literature alone.
One grand challenge in this area is to relate the structure and dynamics of materials as captured by imaging data to desired physical and chemical properties. To date, researchers have predominantly relied on purely qualitative techniques (\textit{e.g.\ }visual inspection) or simple heuristic algorithms to interpret imaging data with this end goal in mind. However, such techniques i) cannot provide a quantitative description of the relationship between the imaging data and material properties and ii) do not scale to large amounts of data. Effectively using large amounts of imaging data to infer material properties requires {\em quantitative} characterization methods for the patterns in microscopy images (uniform regions, defects, \textit{etc.}) that characterize the physical structure of the surface. 

Recently developed methods for quantitative characterization, as shown in Figures \ref{fig:example_pattern}-\ref{fig:example_analysis}, have yielded key fundamental insights into universality of self-assembly dynamics \cite{Harrison2000, Harrison2004, Abukhdeir2008a,Abukhdeir2011a}.
The methods use bond-orientational order theory \cite{Strandburg1992}, and have been applied mainly to studies of block copolymer (BCP) self-assembly on surfaces \cite{Harrison2000,Yokojima2002, Harrison2004, Vega2005}.
They represent first steps toward solving the grand challenge.
Figure \ref{fig:example_analysis} shows an example of such a method: given an image with both a known pattern (hexagonal) and convex pattern features, the method finds orientational relationships among these pattern features.
These relationships can then be used in conjunction with bond-orientational order theory \cite{Strandburg1992} to approximate local pattern orientation and identify defects.
This type of quantification of surface order has been vital in the identification of pattern evolution mechanisms and defect kinetics \cite{Harrison2000,Yokojima2002,Harrison2004,Abukhdeir2008a,Abukhdeir2011a}.
Furthermore, relationships resulting from this type of analysis have since been shown to imply that pattern dynamics are universal, {\em i.e.}\ they are invariant with respect to chemical structure and physical interactions that drive pattern formation \cite{Abukhdeir2008a,Abukhdeir2011a}.
Besides furthering fundamental understanding of self-assembly, these methods will be key enablers of the ultimate goal of controlling self-assembly to produce task-optimized material properties \cite{McGill2014}.

\begin{figure}[h]
    \centering
    \begin{subfigure}{0.45\linewidth}
        \includegraphics[width=\linewidth]{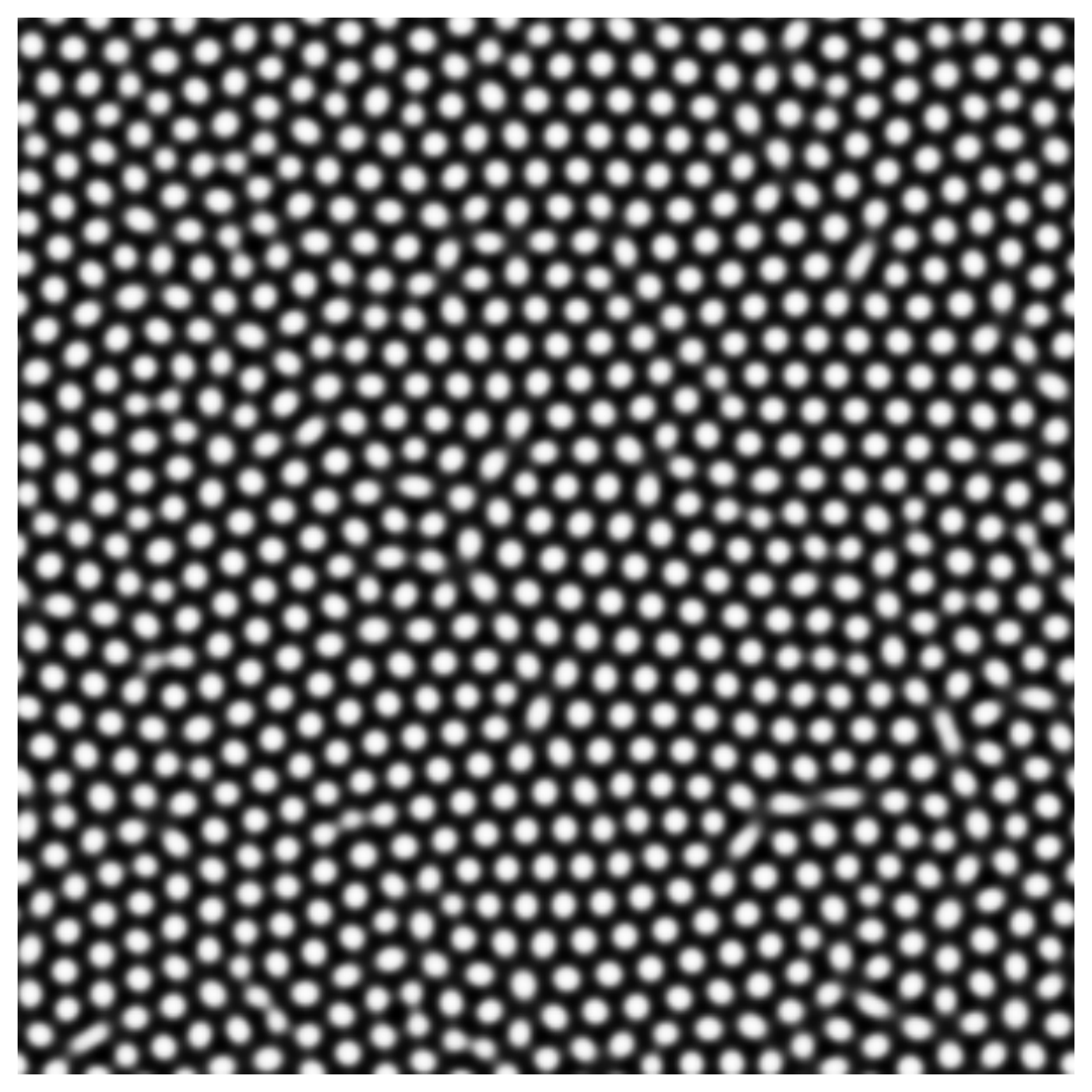}
        \subcaption{}\label{fig:example_pattern}
    \end{subfigure}
    \begin{subfigure}{0.45\linewidth}
        \frame{\includegraphics[width=\linewidth]{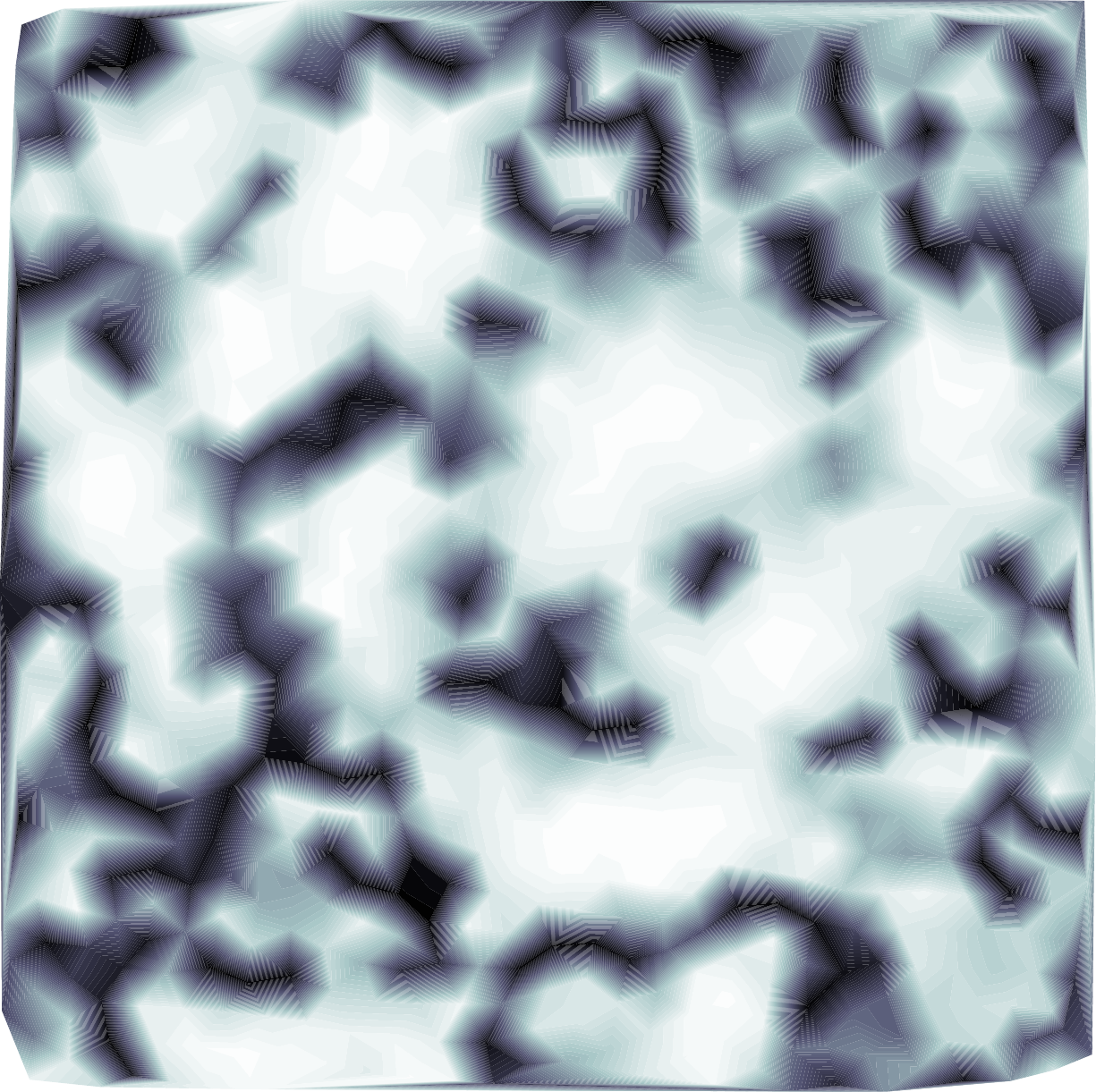}}
        \subcaption{}\label{fig:example_analysis}
    \end{subfigure}
    \caption{(a) Example of a (simulated) hexagonal self-assembled film from past work \cite{Abukhdeir2011} where the field shows surface coverage of a species ranging from $0\rightarrow 1$; (b) local hexagonal order resulting from applying the bond orientational order method to (a) which ranges from no order (black) to perfect order (white).}

\end{figure}

Despite its advantages, the bond-orientational order approach to self-assembly pattern analysis has several key limitations:
\begin{itemize}
    \item \textit{resolution} -- bond-orientational order theory quantifies order at the pattern ``feature'' level, where pattern features are sub-domains which repeat in an ordered way. Typically, these features are larger than the resolution of the image, as is the case in Figure \ref{fig:example_pattern}, which results in a coarse resolution of local pattern order as shown in Figure \ref{fig:example_analysis} where interpolation is used.
    \item \textit{convexity} -- in order to compute \textit{unique} nearest-neighbour ``bonds'' between pattern features, the features must be convex. This precludes the use of the method on stripe patterns and patterns in which features vary greatly in character (strained patterns).
    \item \textit{uncertainty} -- in order to compute nearest-neighbour ``bonds'' between pattern features, features must be uniquely identifiable. Typical experimental images of self-assembly phenomena involve nanoscale features which result in significant measurement uncertainty.
\end{itemize}
Furthermore, images frequently contain multiple regions that may or may not contain patterns, and that sometimes contain multiple patterns.
Thus there is a clear need for robust, automated approaches to pattern recognition (``Is a pattern present in this image? Where within the image?'') and classification (``What type of pattern is present?'') for self-assembly imaging \cite{Rehse2008}, in addition to a more detailed characterization (``How is the pattern oriented? Where are the defects?'') once these initial questions have been answered.

This work presents and demonstrates an analysis method for self-assembled surface imaging that is fundamentally different from past approaches for BCPs and that addresses the limitations described above.
The basis of the method is a family of localized functions called {\em shapelets} \cite{Refregier2003}.
Shapelets were originally developed to characterize images of galaxies ($\propto \SI{e20}{\meter}$) \cite{Refregier2003}; they are used here to characterize images of nanopatterned surfaces ($\propto \SI{e-9}{\meter}$).
It is demonstrated that, using simulation data of self-assembled surfaces, the presented approach is able to robustly determine \emph{local} pattern characteristics, using an appropriate subset of shapelets \cite{Refregier2003,Massey2005} and steerable filter theory \cite{Freeman1991}, such as sub-domains that are well-ordered, strained, and/or have defects present.

\section{Background}

The analysis method synthesizes global pattern information derived from the discrete Fourier transform with local pattern information derived from projecting the image onto shapelet filters in a rotation-invariant way using steerable filter theory. Fourier analysis, shapelets, and steerable filter theory are reviewed below.

\subsection{Fourier Analysis}
\newcommand{\diff}{\mathrm{d}}

The discrete Fourier transform (DFT) is a standard image analysis approach that can be used to quantify the presence of patterns or periodicity in an image. Given an image with intensity given by $f(x,y)$, the DFT of $f$ is  
\newcommand{\ee}{\ensuremath{\mathrm{e}}}
\newcommand{\I}{\ensuremath{\mathrm{\imath}}}
\begin{equation}
\mathcal{F}\{f\}(u,v) = \sum_{x = 0}^{X-1} \sum_{y = 0}^{Y - 1} f(x,y)~\exp{-\I\left( \frac{2\pi ux}{X} + \frac{2\pi vy}{Y}\right)} \label{eq:dft}
\end{equation}
which transforms the image from the spatial domain to the frequency domain.
The resulting coefficients $\mathcal{F}\{f\}(u,v)$ of the Fourier modes over a discrete set of wave vectors indexed by $u$ and $v$ characterize both wavelength and orientation \cite{Szeliski2011} of all periodic image patterns.
While the DFT can be computed efficiently, the coefficients provide only {\it global} information in the sense that the Fourier modes span the whole spatial domain, \ie, they are not localized in space.
Thus, this decomposition can only determine the presence of domain-wide periodic components in the image and recover their characteristic wavelengths.

Figure \ref{fig:dft} shows the resulting frequency domain representation of the image from Figure \ref{fig:example_pattern}, and Figure \ref{fig:dftavg} shows the radially averaged spectral density.
For images of simple uniformly-ordered domains (e.g.\ those with a single orientation) the DFT provides sufficient information about the type of pattern and its orientation to fully characterize the pattern.
However, this simplicity is rarely observed in experimental imaging of self-assembled surfaces.
In domains that are not well-ordered, the output of the DFT reveals only the presence of periodic structure within the image and characteristic length scales of that structure; it does not reveal {\em local} pattern structure.
For example, the peaks in the radially averaged spectral density visible in Figure \ref{fig:dftavg} reveal length scales of periodic patterns in Figure \ref{fig:example_pattern}, but neither they nor the full DFT in Figure \ref{fig:dft} reveal the location of defects or grain boundaries.

\begin{figure}[h]
    \centering
    \begin{subfigure}{0.3\textwidth}
        \includegraphics[width=\textwidth]{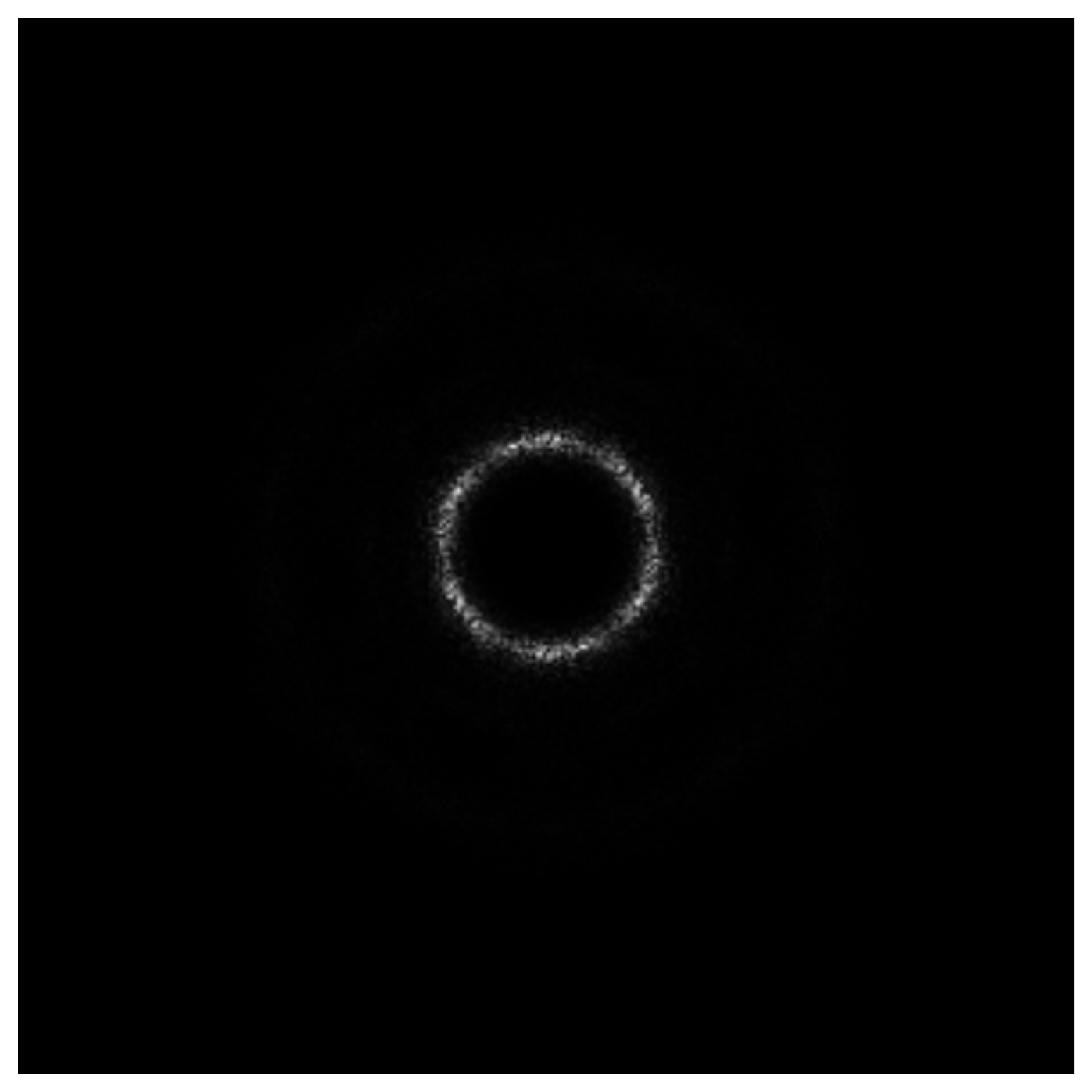}
        \subcaption{}\label{fig:dft}
    \end{subfigure}
    \begin{subfigure}{0.45\textwidth}
        \includegraphics[width=\textwidth]{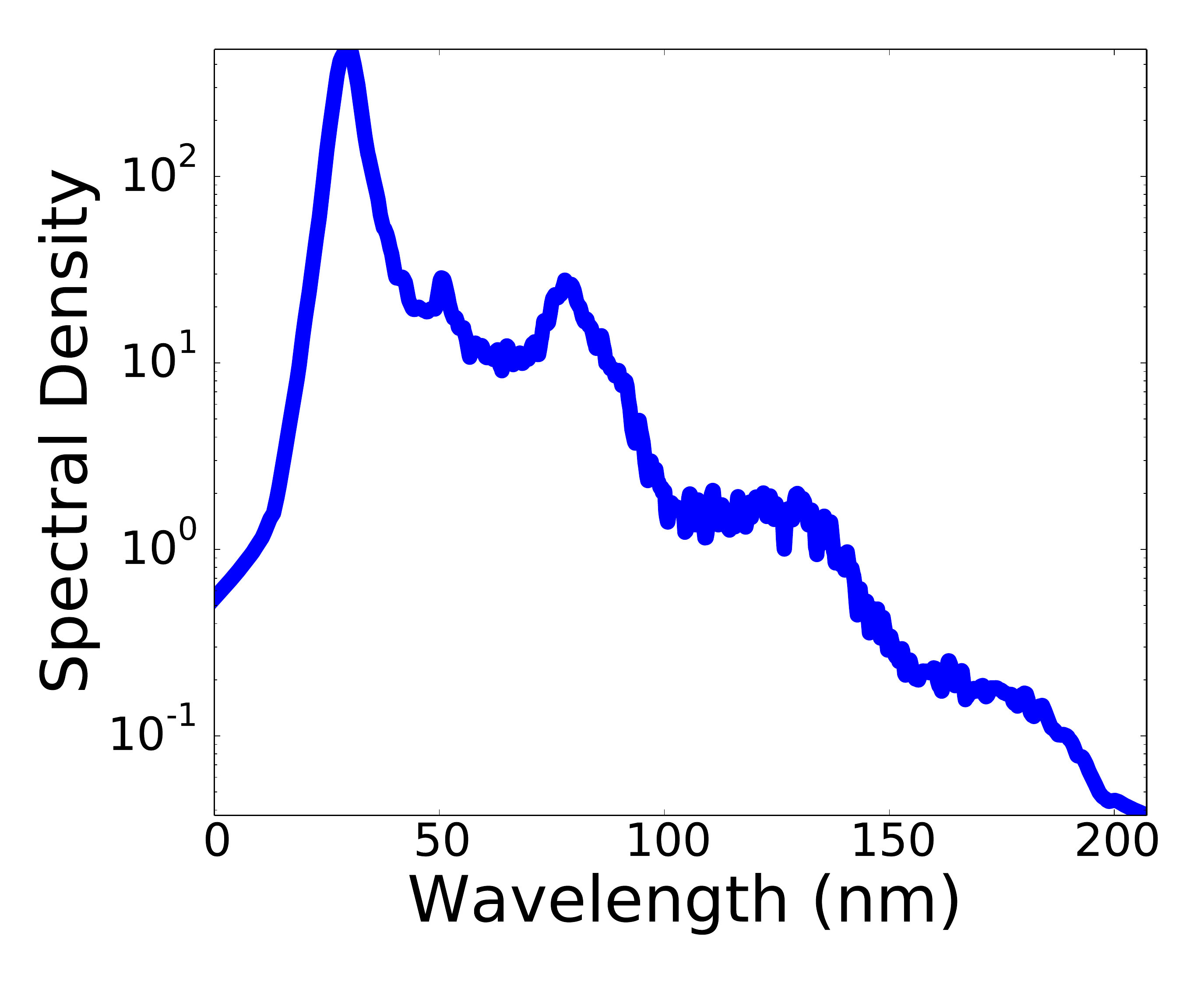}
        \subcaption{}\label{fig:dftavg}
    \end{subfigure}

    \caption{(a) Two-dimensional spectral density plot from the DFT of Figure \ref{fig:example_pattern} with the origin at the centre of the image; (b) radially averaged spectral density from (a).}

\end{figure}

\subsection{Shapelets}
Shapelets are a recently-proposed family of orthonormal basis functions which have been used for image analysis and are particularly suited to characterizing local pattern features.
Shapelet analysis projects the image of interest onto basis functions of \emph{fixed scale} and \emph{varying shape}, shown in Figure \ref{fig:shapelets}.
As with the windowed Fourier transform and wavelet decompositions \cite{book}, shapelet analysis \cite{Refregier2003} involves a linear decomposition of an image obtained by projecting it onto a set of localized orthonormal basis functions (Figure \ref{fig:shapelets}).
The {\it polar shapelets} \cite{Massey2005} are of particular interest as they possess rotational symmetries that are also present in images of self-assembled materials. They are given by,
\begin{equation}\label{eqn:2D_polar_shapelets}
B^{\circlearrowleft}_{n,m}(r,\theta;\beta) = \beta^{-1}~\chi_{n,m}\left(\beta^{-1}r\right) \ee^{-\I m \theta}
\end{equation}
where $n$ and $m$ are nonnegative integers, and $\beta$ is a characteristic length scale. The function $\chi_{n,m}$ is given by
\begin{equation}\label{eqn:1D_polar_shapelets}
\chi_{n,m}(r) = c_{n,m} r^m L^{m}_{n}(r^2)\ee^{\frac{-r^{2}}{2}}
\end{equation}
where $c_{n,m}$ are constants, and $L^{m}_{n}(r) = \frac{r^{-m} \ee^{r}}{n!} \frac{\diff^n}{\diff r^n} (r^{m+n} \ee^{-r})$ are the associated Laguerre polynomials.\footnote{Massey et al.\ use alternative indices $m',n'$ where $m' = m$ and $n' = 2n + m$. This requires additional constraints on $m',n'$.}
Examples for $n = 0$ and $m$ from $0$ through $6$ are shown in Figure \ref{fig:shapelets}. Note that for $m > 0$ a polar shapelet has a real part and an imaginary part that is equal to a rotation of the real part by an angle of $-\pi/(2m)$.
It is natural to define polar shapelets in polar coordinates; however, for convenience, define $B_{n,m}(x,y;\beta) \triangleq B^\circlearrowleft_{n,m}(x^2+y^2,\tan^{-1}(y,x);\beta)$ since images are expressed in Cartesian coordinates.

Standard shapelet analysis of an image is similar to that of other discrete transforms such as the Fourier and wavelet transforms. The image is decomposed into a linear combination of the basis functions,
\begin{equation}\label{eqn:linear_comb}
f(x,y) = \sum_{n=0}^N\sum_{m=0}^M w_{n,m}B_{n,m}(x,y;\beta)
\end{equation}
where each weighting coefficient $w_{n,m}$ is given by a discrete inner product or {\it correlation} of its corresponding shapelet $B_{n,m}(x,y;\beta)$ with the image function $f(x,y)$,
\begin{equation}\label{eqn:inner_product} 
w_{n,m} = f \star B_{n,m} \triangleq \sum_{x'}\sum_{y'} f(x',y')B_{n,m}(x',y';\beta).
\end{equation} 
The real (resp.\ imaginary) part of coefficient $w_{n,m}$ can be interpreted as a measure of similarity between the image and the real (resp.\ imaginary) part of $B_{n,m}$. The coefficient $w_{n,m}$ is termed the {\it response} for shapelet $B_{n,m}$. Because the shapelets are spatially localized functions, different translated versions of a shapelet have different responses. It is therefore natural to consider response as a function of image coordinates $x$ and $y$
\begin{equation}\label{eqn:inner_product} 
w_{n,m}(x,y) = \sum_{x'}\sum_{y'} f(x',y')B_{n,m}(x'-x,y'-y;\beta)
\end{equation} 
which gives the similarity between the image and a shapelet whose origin has been translated to image location $(x,y)$.

\begin{figure}[h]
\centering
  \includegraphics[width=0.11\linewidth]{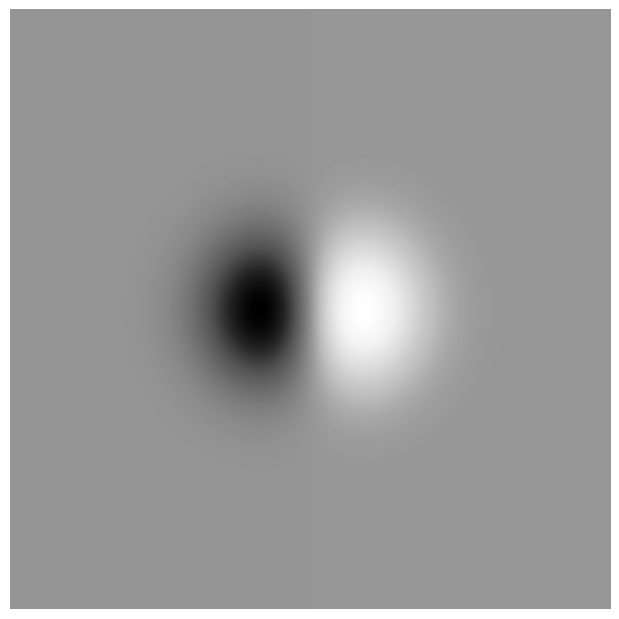}
  \includegraphics[width=0.11\linewidth]{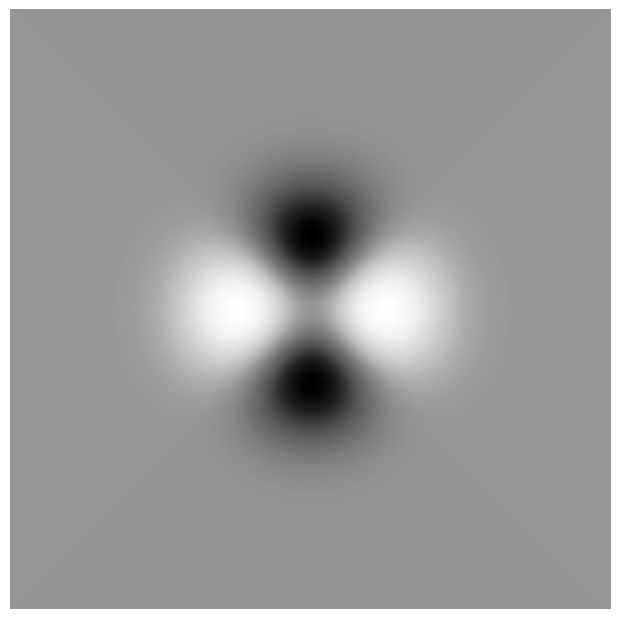}
  \includegraphics[width=0.11\linewidth]{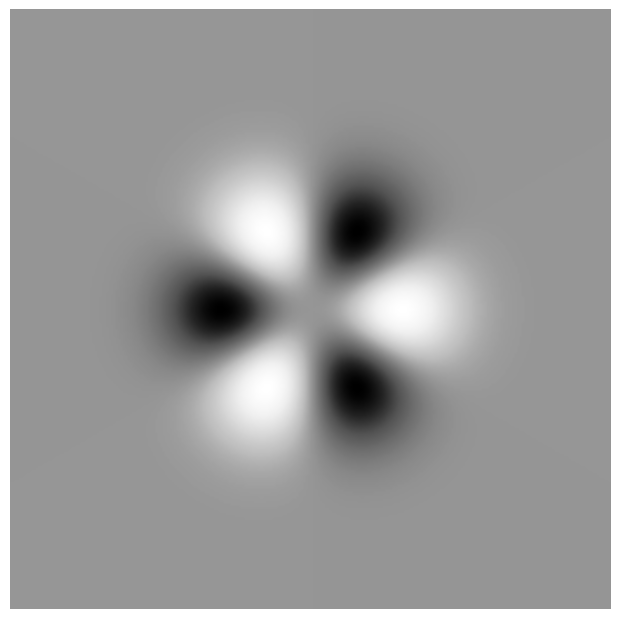}
  \includegraphics[width=0.11\linewidth]{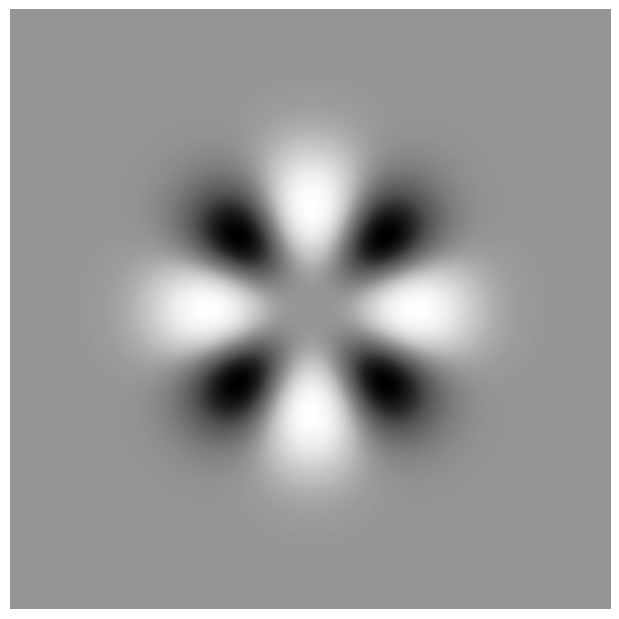}
  \includegraphics[width=0.11\linewidth]{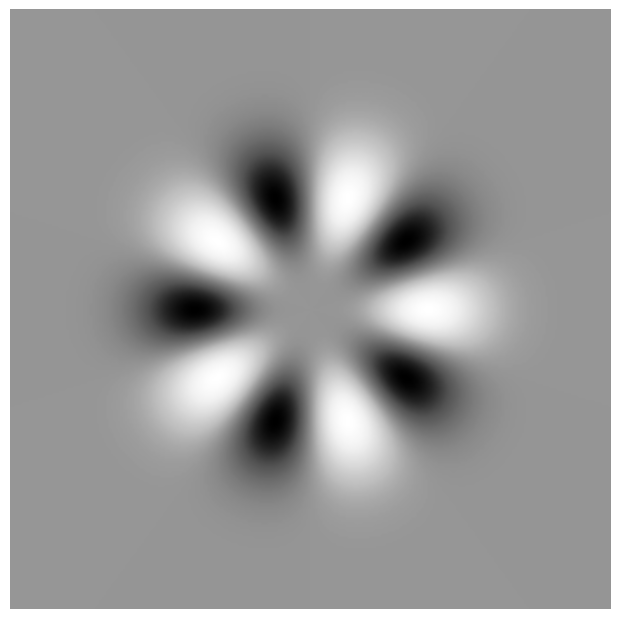}
  \includegraphics[width=0.11\linewidth]{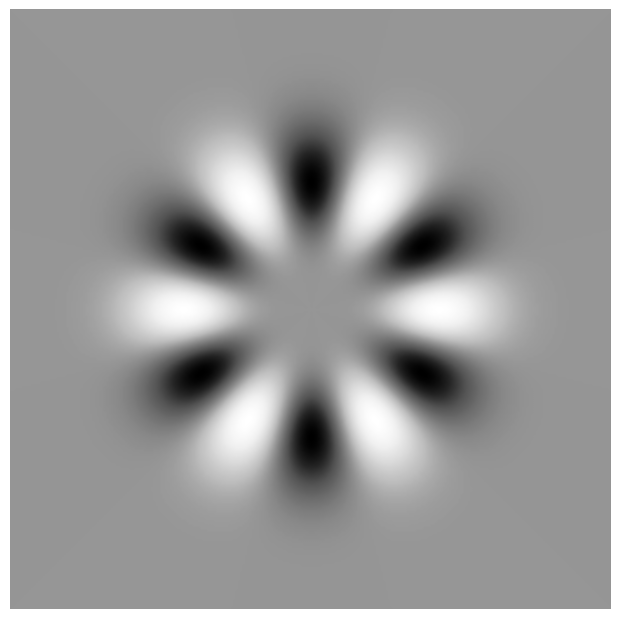}
  \includegraphics[width=0.11\linewidth]{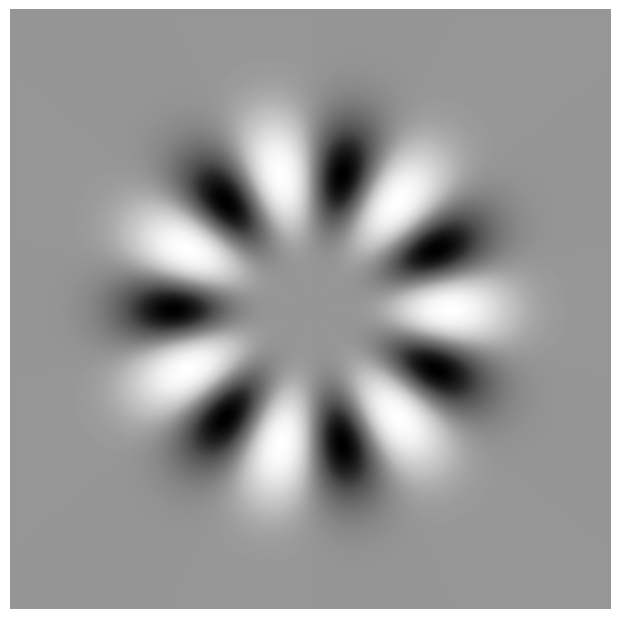}
  \includegraphics[width=0.11\linewidth]{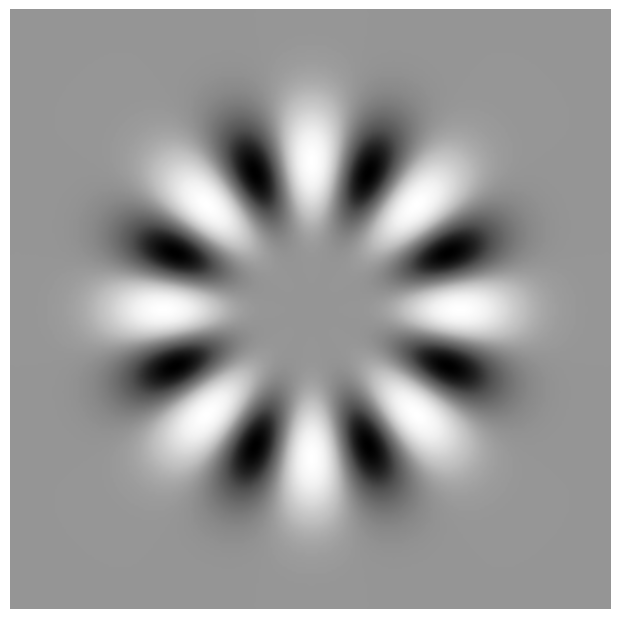}\\
  \includegraphics[width=0.11\linewidth]{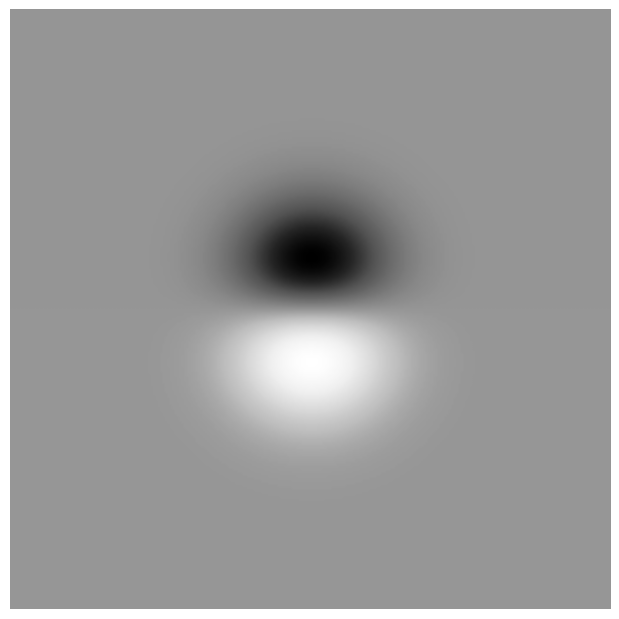}
  \includegraphics[width=0.11\linewidth]{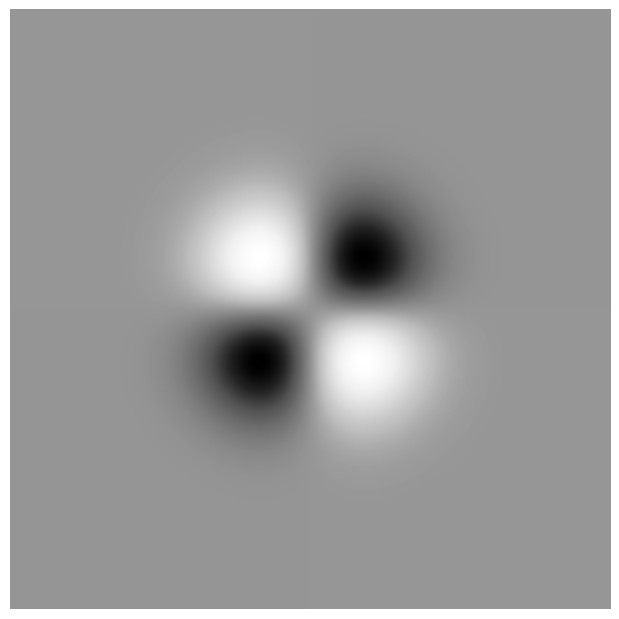}
  \includegraphics[width=0.11\linewidth]{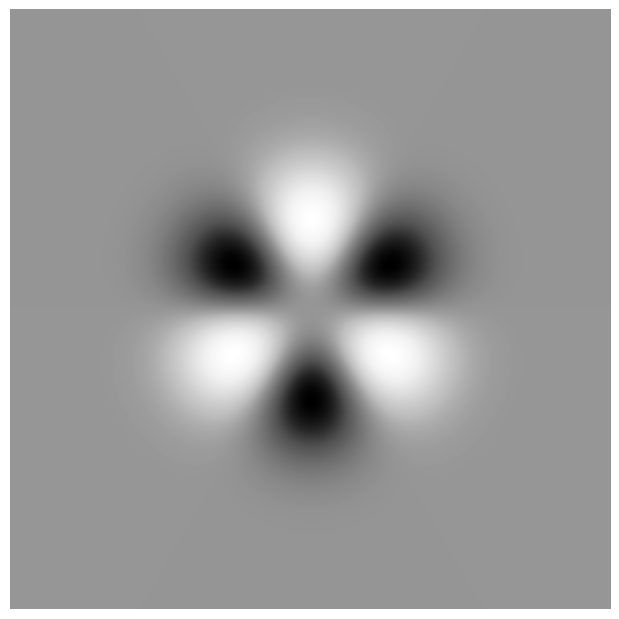}
  \includegraphics[width=0.11\linewidth]{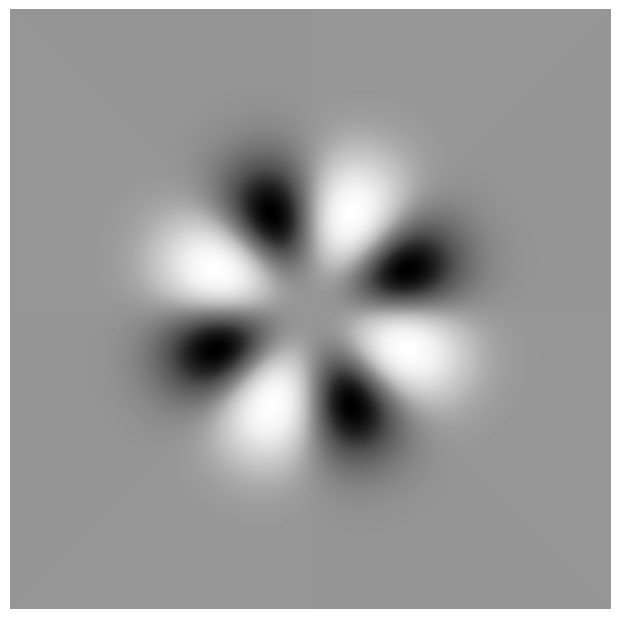}
  \includegraphics[width=0.11\linewidth]{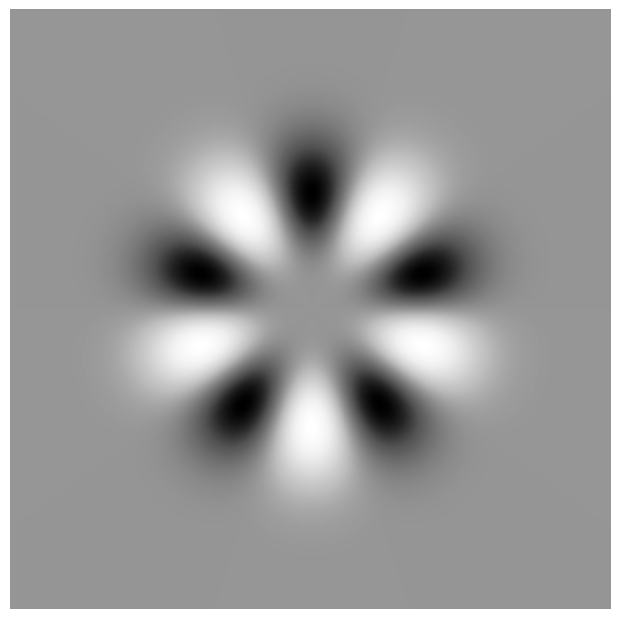}
  \includegraphics[width=0.11\linewidth]{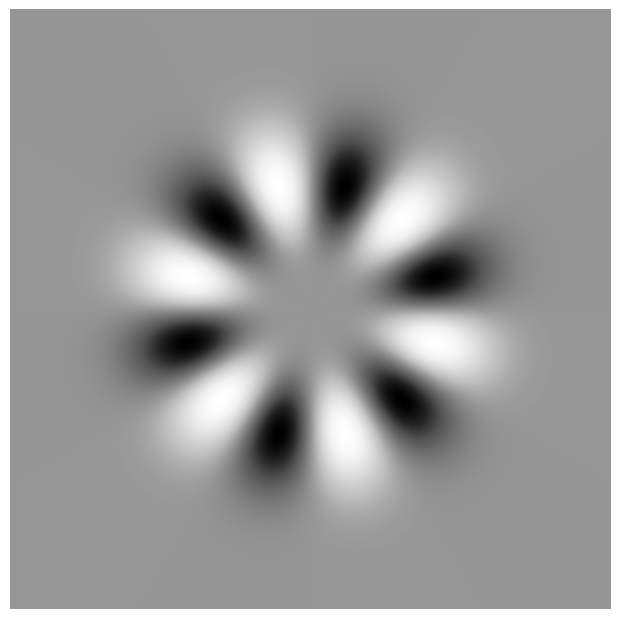}
  \includegraphics[width=0.11\linewidth]{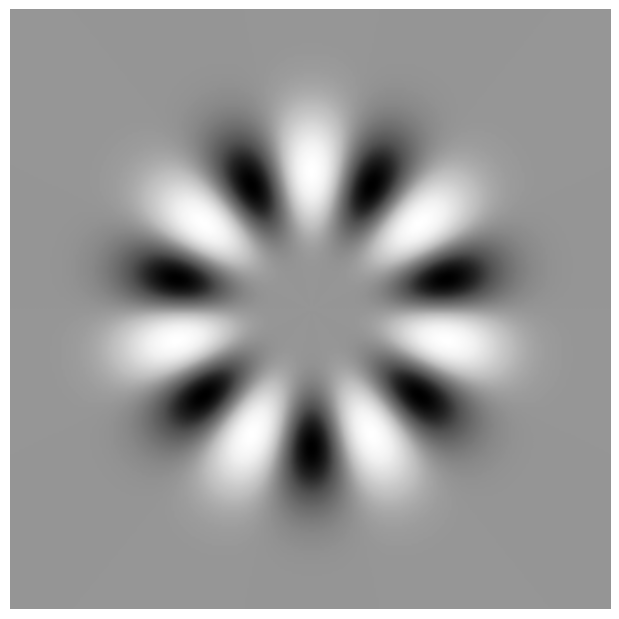}
  \includegraphics[width=0.11\linewidth]{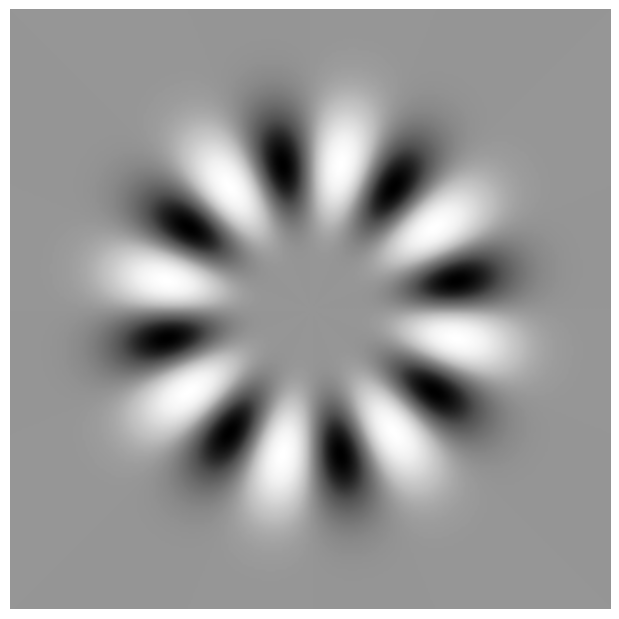}

\caption{Surface plots of the real (top row) and imaginary (bottom row) components of polar shapelet functions (eqn. \ref{eqn:2D_polar_shapelets}) with $n=0$ and $m=1\rightarrow 8$. Image length-scale is normalized to $1$ and $\beta=0.55$.}
  \label{fig:shapelets}
\end{figure}

\subsection{Steerable Filter Theory}\label{ss:steerable}
\newcommand{\rotang}{\varphi}

Given a (typically real-valued) filter $F(x,y)$ expressed in Cartesian coordinates, it is often useful to know how that filter would respond if it (or the image) were rotated.
Consider a version of the filter that has been rotated clockwise by a phase angle $\rotang$, $F(x,y;\rotang) = F(x \cos \rotang - y \sin \rotang,x\sin \rotang + y \cos \rotang)$.
It is particularly useful to know the $\rotang$ for which the filter response is maximal, as this angle contains information about pattern orientation: it gives the angle by which the pattern (or the filter) must be rotated in order to achieve maximum similarity between the pattern and the filter. 
Freeman \cite{Freeman1991} shows if a filter is {\it steerable}, any rotated version of the filter can be expressed as a linear combination of a finite (and typically small) set of filters.
This allows exact computation of the $\rotang$ for which the filter response is maximal much more efficiently than brute-force techniques, which must explicitly compute a large number of filter orientations \cite{Wang2013}.

\section{Results and Discussion}

The successful shapelet-based analysis of self-assembly image data requires methods for determining an appropriate subset of shapelets, the optimal orientation of the shapelets, and the appropriate scales for the selected shapelets.
Section \ref{sec:shapeletselection} develops methods for each of these tasks, and Section ~\ref{sec:application} demonstrates the use of the developed methods on self-assembly data.

\subsection{Sets of Steerable, Scale-Optimized Shapelets}\label{sec:shapeletselection}

\newcommand{\T}{\ensuremath{\mathsf{T}}}
\newcommand{\ptnconst}{a}

In order to use shapelets for pattern analysis, appropriate shapelet sets and optimal scales are first determined based on ``prototypical'' uniform patterns that approximate real surface self-assembly imaging data, but that have a convenient parametric form. The uniform patterns are expressed in terms of a two-dimensional Fourier series \cite{Gunaratne1994},
\begin{equation}\label{eq:uniformpattern}
  \rho(\bm{x}) = \sum_{n=0}^{N} \ptnconst_{n}\exp{\left(i\bm{k}_{n}\cdot\bm{x}\right)}
\end{equation}
where the constants $\ptnconst_{n}$ are related to the magnitude of the pattern modulation and $\bm{k}_{n}$ are the basis vectors of the pattern.
For one-mode approximations of stripe and hexagonal patterns of interest, the basis vectors are \cite{Gunaratne1994},
\begin{align}\label{eqn:uniform_coeffs}
    \bm{k}_{1} & = \frac{2\pi}{\lambda}\;\bm{e}_{2}\nonumber\\
    \bm{k}_{2} & = \frac{2\pi}{\lambda}\left(\frac{\sqrt{3}}{2}\bm{e}_{1} - \frac{1}{2}\bm{e}_{2} \right)\\
    \bm{k}_{3} & = \frac{2\pi}{\lambda}\left(\frac{-\sqrt{3}}{2}\bm{e}_{1} - \frac{1}{2}\bm{e}_{2} \right).\nonumber
\end{align}
where $\bm{e}_{1} = (1,0)^\T$ and $\bm{e}_{2} = (0,1)^\T$. For a stripe pattern, $\ptnconst_{1} \ne 0, \ptnconst_{2} = \ptnconst_{3} = 0$, and for a hexagonal pattern $\ptnconst_{1} = \ptnconst_{2} = \ptnconst_{3} \ne 0$.
The quantity $2\pi/\lambda$ is the wavenumber for the pattern length scale (wavelength) $\lambda$, which corresponds to the peak shown in Figure~\ref{fig:dft}.

In order to select a minimal set of shapelets that respond strongly when applied to uniform stripe and hexagonal patterns, or in general any surface pattern, the sub-set of shapelets should have the following properties:
\begin{enumerate}
    \item It should contain shapelets with the same fundamental rotational symmetries as the pattern of interest. Stripe patterns have subregions with 1- and 2-fold symmetry; hexagonal patterns have subregions with 1-, 2- , 3-, and 6-fold symmetry.
    \item The response magnitude of the shapelets should be \emph{invariant} with respect to rotations of the pattern. 
    \item The shapelets should respond most strongly to the dominant pattern frequencies.    
\end{enumerate}

\subsubsection{Pattern Symmetries} Referring to Figure \ref{fig:shapelets}, a convenient property of shapelets of order $(m,n)$ with $n=0$ and $m > 0$ is that they have $s$-fold rotational symmetries for $s$ corresponding to all numbers that divide $m$, plus the trivial 1-fold symmetry.
Thus, at minimum, shapelets up to and including order $m=2$ are necessary for analysis of stripe patterns, and shapelets up to order $m=6$ are necessary for hexagonal patterns.
For simplicity, shapelets up to order $m=6$ are used in all example analyses, which form an overcomplete set of shapelets for stripe patterns and a minimal set for hexagonal patterns. The ``redundant'' information that this set provides in the striped pattern case did not prove to be problematic.

\subsubsection{Rotational Invariance} In order to produce an analysis that is invariant to pattern rotations, for each shapelet in the set, the rotation of the shapelet that produces the largest response is determined.
This is illustrated in Figure \ref{fig:steerable_example}.
The brute-force, and thus computationally inefficient, approximate approach to finding the optimal rotation would be to determine the shapelet response for a large number of rotations of the shapelet and then select the rotation with the maximal response.
This has two drawbacks: (i) the large number of shapelet responses must be evaluated at each pixel and (ii) the solution is not exact.

\begin{figure}[h]
    \centering
    \begin{subfigure}{0.3\linewidth} 
        \includegraphics[width=\linewidth]{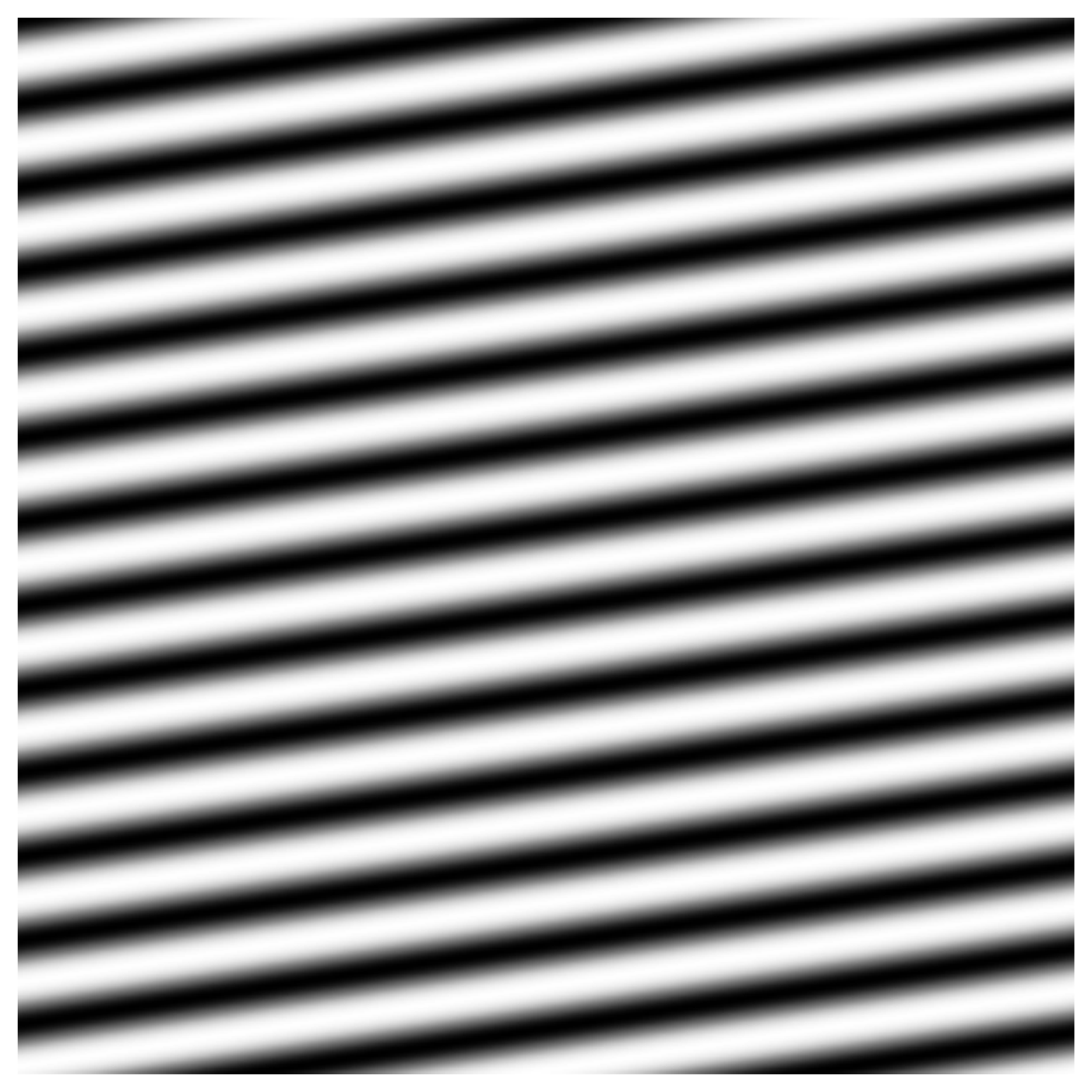}
        \subcaption{}\label{fig:perfect_stripe}
    \end{subfigure}
    \begin{subfigure}{0.3\linewidth} 
        \includegraphics[width=\linewidth]{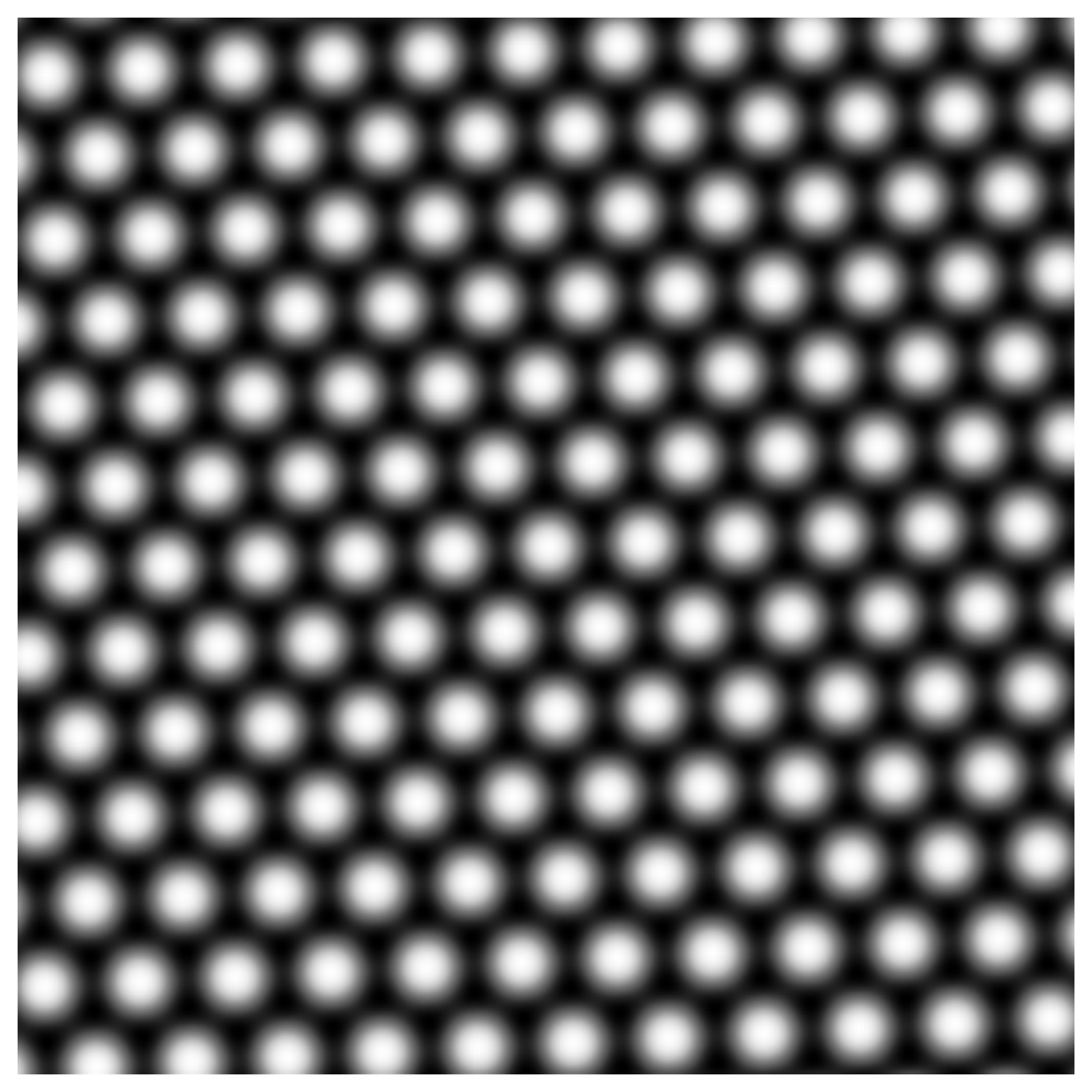}
        \subcaption{}\label{fig:perfect_hexagonal}
    \end{subfigure}
    \begin{subfigure}{0.3\linewidth} 
        \includegraphics[width=\linewidth,clip,trim=175px 25px 0px 150px]{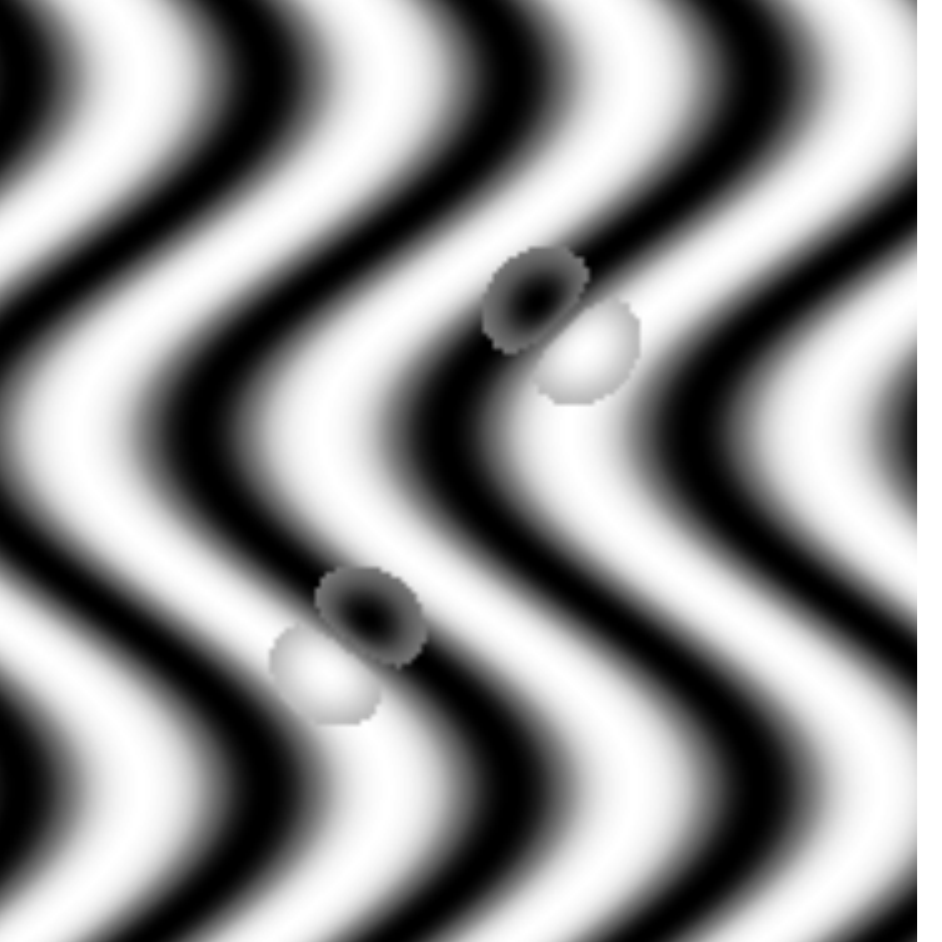}
        \subcaption{}\label{fig:steerable_example}
    \end{subfigure}
  
\caption{Plots of one-mode approximations of a (a) stripe and (b) hexagonal pattern using eqns. \ref{eq:uniformpattern}-\ref{eqn:uniform_coeffs}; (c) Schematic of rotations of a shapelet applied to a nonuniform stripe pattern.}
  \label{fig:steerable}
\end{figure}
Rather than solve for the optimal rotation approximately, an exact solution is derived from the fact that the shapelets are steerable.
\begin{lemma}\label{thm:steerable_shapelet} Let $B_{n,m}(x,y;\beta,\varphi) = B_{n,m}(x \cos \rotang - y \sin \rotang,x\sin \rotang + y \cos \rotang;\beta)$ be a shapelet as defined in {\rm (\ref{eqn:2D_polar_shapelets})} that has been rotated clockwise through a phase angle $\varphi$ as described in Section~\ref{ss:steerable}. Then 
\begin{equation}\label{eqn:steerable_shapelet}
\Re[B_{n,m}(x,y;\beta, \rotang)] = \cos{(m\rotang)}~\Re[B_{n,m}(x,y;\beta)] + \sin{(m\rotang)}~\Im[B_{n,m}(x,y;\beta)].
\end{equation}
\end{lemma}
\begin{proof}
(Sketch) Note from the polar form of shapelets that $B^{\circlearrowleft}_{n,m}(r,\theta;\beta)$ is of the form $h(r) \ee^{-\I m \theta}$, and that a version rotated by an angle $\varphi$ is of the form $h(r) \ee^{-\I m (\theta + \varphi)}$. The lemma follows from applying standard trigonometric identities to the complex exponential part of $B^{\circlearrowleft}$ and simplifying. 
\end{proof}
Steerable forms of each shapelet were formulated using Lemma~\ref{thm:steerable_shapelet} (eqn.\ \ref{eqn:steerable_shapelet}),
which are shown for $\rotang=0$ in Figure \ref{fig:shapelets} (highlighting of background figure). The steerable forms were then used to determine the optimal rotation angle for each shapelet at each image location.
\begin{lemma}
Let $w_{0,i} = f \star B_{0,i}(\cdot,\cdot;\beta)$, and define $\rotang^*_{0,i} = \arg\max_\varphi \Re[f \star B_{0,i}(\cdot,\cdot;\beta,\varphi)]$ and $w^*_{0,i} = \Re[f \star B_{0,i}(\cdot,\cdot;\beta,\rotang^*_{0,i})]$. Then 
\begin{equation}\label{eqn:steerable_orientation}
\rotang^*_{0,i} = \arg{w_{0,i}},~~~ w^*_{0,i} = |w_{0,i}|.
\end{equation}
\end{lemma}
\begin{proof}(Sketch.)
Since $\Re[B_{n,m}(x,y;\beta, \rotang)]$ is continuous in $\rotang$, it suffices to take the derivative, equate it to zero, and check second-order optimality conditions to find the optimal rotation and magnitude. Note that $\arg{w_{0,i}}$ is one of a countably infinite set of solutions to the optimal rotation problem.
\end{proof}
Here, $\rotang^*_{0,i}$ is the shapelet orientation at which the real part of the \emph{steered} shapelet response is maximal and $w^*_{0,i}$ is the value of the response at that orientation.
 In the above, the dependence of $\rotang^*_{0,i}$ and $w^*_{0,i}$ on $x$, $y$, and $\beta$ is suppressed in the notation for clarity, but the lemma immediately applies to translated and scaled versions of shapelets as well. As desired, the rotation-optimized response $w^*_{0,i}$ is invariant to rotations of the pattern.

 \subsubsection{Scaling} The scale of a shapelet, given by the parameter $\beta$, also affects shapelet response. In order to ensure that the selected shapelets respond strongly to the pattern of interest and therefore do {\em not} respond strongly to pattern defects, their length scales are tuned to the dominant pattern frequency.  Fixing a pattern, location $(x,y)$, and optimal orientation $\rotang^*$ of a shapelet, the shapelet response is given by the correlation of the shapelet with the image function $f$, \begin{align}
  w^*_{n,m}(\beta) = f \star \Re[B_{n,m}(\cdot,\cdot;\beta,\rotang^*)] \triangleq \sum_{x'}\sum_{y'} f(x',y')\Re[B_{n,m}(x',y';\beta,\rotang^*)],
\end{align}
which is only a function of $\beta$; examples for different $(n,m)$ are shown in Figure~\ref{fig:shapelet_scale}.

\begin{figure}
    \centering
   \begin{subfigure}{0.45\linewidth}
        \centering
        \includegraphics[width=\linewidth]{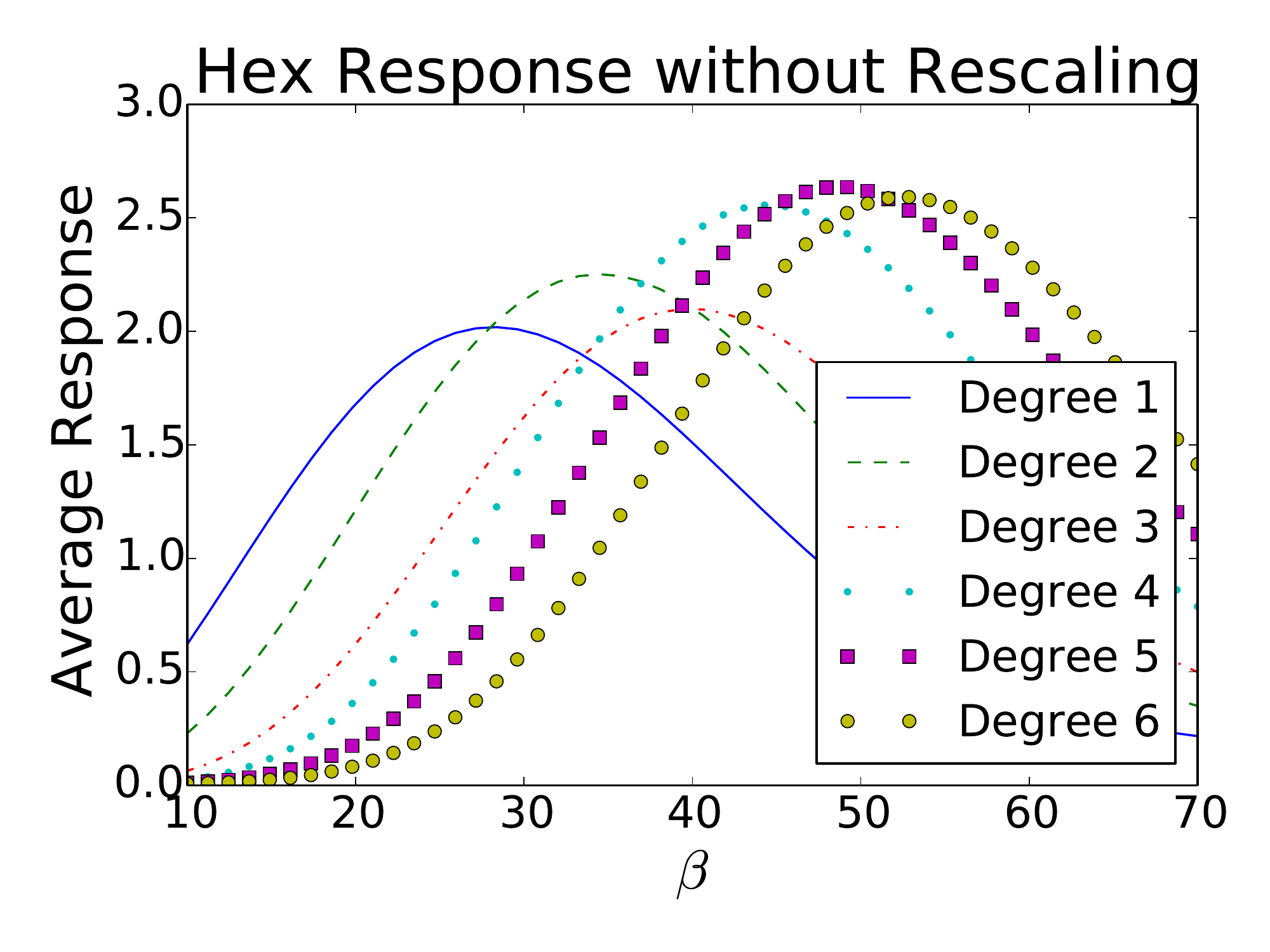}\\

        \subcaption{}\label{fig:unscaled}
    \end{subfigure}    
   \begin{subfigure}{0.45\linewidth}
        \centering
        \includegraphics[width=\linewidth]{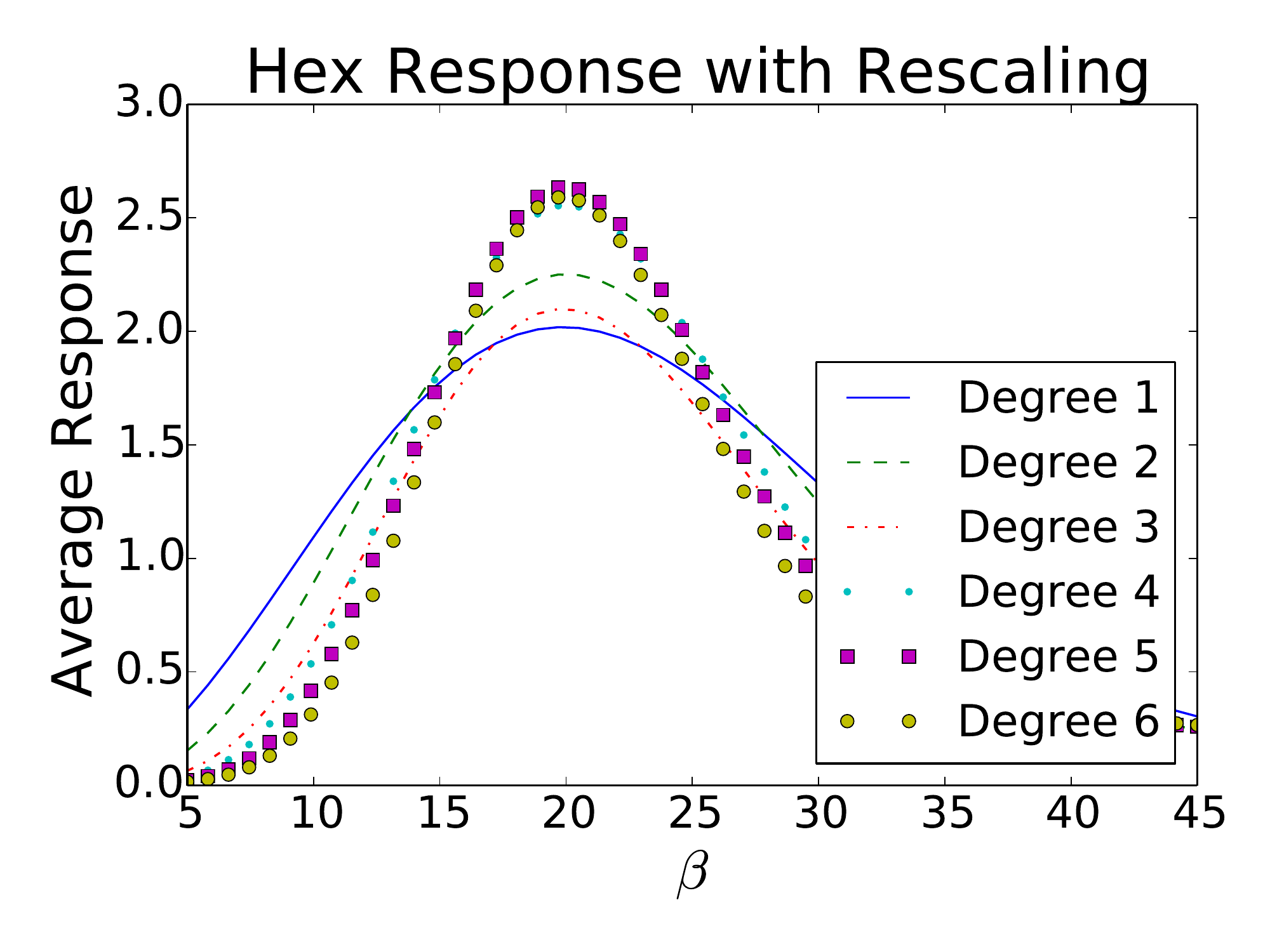}\\

        \subcaption{}\label{fig:rescaled}
    \end{subfigure}    
   \caption{Shapelet filter response versus $\beta$ (a) before and (b) after rescaling of $\beta$ with respect to $\lambda$. Responses shown are correlations with the uniform pattern given by eqns. \ref{eq:uniformpattern}-\ref{eqn:uniform_coeffs}.}
  \label{fig:shapelet_scale}
\end{figure}

In order to ensure that the chosen shapelets respond optimally to the target pattern of interest, $\beta$ is chosen so as to maximize the function $w^*_{C_1}$ (and $w^*_{C_2}$ and so on).
For a given pattern length scale $\lambda$, the maximal shapelet response was found to {\em not} be at $\beta = \lambda$, but rather at $\beta = C\lambda$, with $C$ depending on the order of the shapelet.
Appropriate constants $C$ were found for each shapelet using grid search on $R$; these are given in Table \ref{tab:beta}.
All analyses that follow use this rescaling.

\begin{table}
    \caption{Coefficient values for $\beta = C\lambda$ for shapelets up to order $6$.}\label{tab:beta}
    \centering{}
    
    \begin{tabular}{|c|c|c|c|}
    \hline
    \textbf{(m,n)} & $C$ & \textbf{(m,n)} & $C$\\
    \hline
    (1,0) & 1.418 & (2,0) & 1.725\\
    (3,0) & 2.003 & (4,0) & 2.224\\
    (5,0) & 2.439 & (6,0) & 2.640\\
    \hline        
    \end{tabular}
\end{table}

Figure \ref{fig:steerable_uniform} shows rotation-optimized local pattern response values resulting from application of the scale-optimized steerable shapelets to uniform (i) stripe and (ii) hexagonal patterns using the uniform pattern given by eq.\ \ref{eq:uniformpattern}.
The shapelet responses shown in Figure \ref{fig:stripe_response} and Figure \ref{fig:hex_response} reveal the locations in the pattern that have different rotational symmetries. For example, the $m=1$ shapelet applied to the striped pattern reveals where the pattern has only 1-fold symmetry: this occurs at the boundaries between stripes. At such locations, the pattern is only self-similar when rotated through a full $2\pi$. The $m=2$ shapelet, on the other hand, responds at peaks and troughs in the stripe pattern, where a rotation of $\pi$ results in a self-similar pattern. As the shapelet order is increased beyond $m = 2$, almost no new local pattern information is extracted; the rotation-optimized responses of the shapelets with higher-order symmetry are very similar to the $m = 1$ or $m = 2$ case.

The $m=1$ response applied to the hexagonal pattern responds at locations near the ``edge'' of a pattern mode, which have only 1-fold symmetry. The $m=2$ and $m=4$ shapelets respond strongly to areas that are midway along a line joining two pattern modes; these locations have 2-fold symmetry. The $m=3$ shapelet responds very strongly to the 3-fold symmetry at ``saddle points''---that is, at points in the pattern that are equidistant from three nearby modes, and finally the $m=6$ shapelet responds strongly at pattern modes where there is 6-fold rotational symmetry. Note that the hexagonal pattern lacks any 5-fold symmetry; thus the $m=5$ shapelet, which has only 5-fold and 1-fold symmetry responds only at pattern locations with 1-fold symmetry which are also identified by the $m=1$ shapelet.

The next section explains how the responses of different shapelets can be combined into a useful quantitative analysis of the underlying image.

\begin{figure}
    \centering
    \begin{subfigure}{\linewidth}
        \centering
        \tabcolsep=0.1em
        \begin{tabular}{cccccc}
        $m = 1$ & $m = 2$ & $m = 3$ & $m = 4$ & $m = 5$ & $m = 6$ \\
        \includegraphics[width=0.15\linewidth]{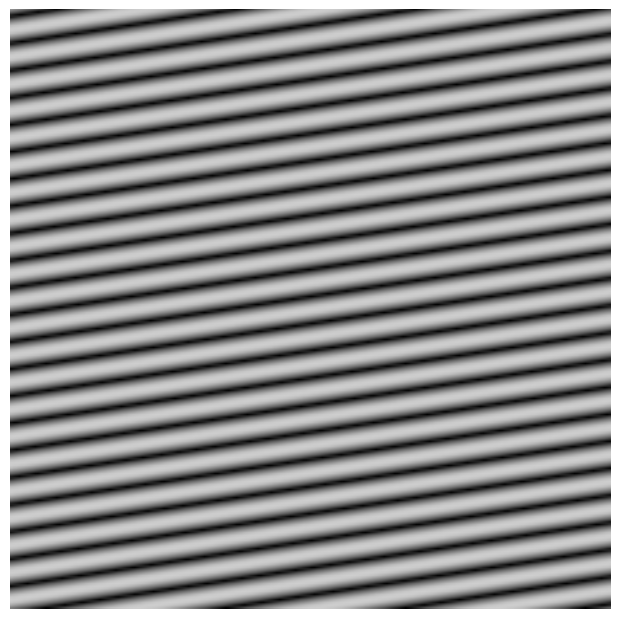} &
        \includegraphics[width=0.15\linewidth]{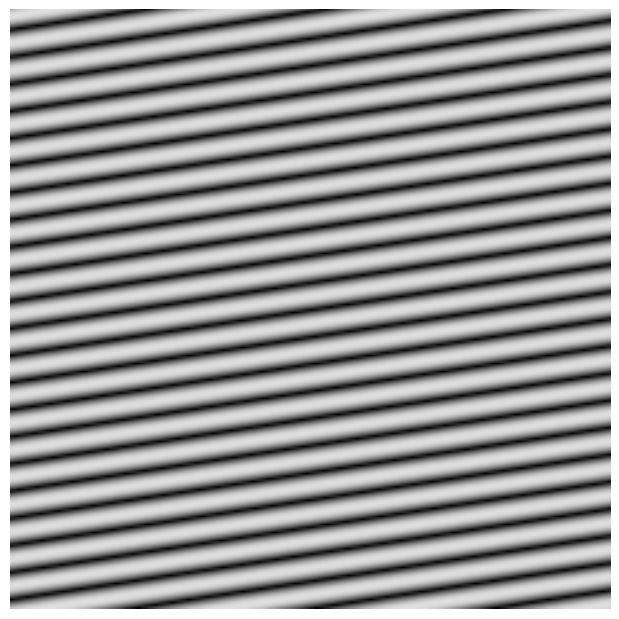} &
        \includegraphics[width=0.15\linewidth]{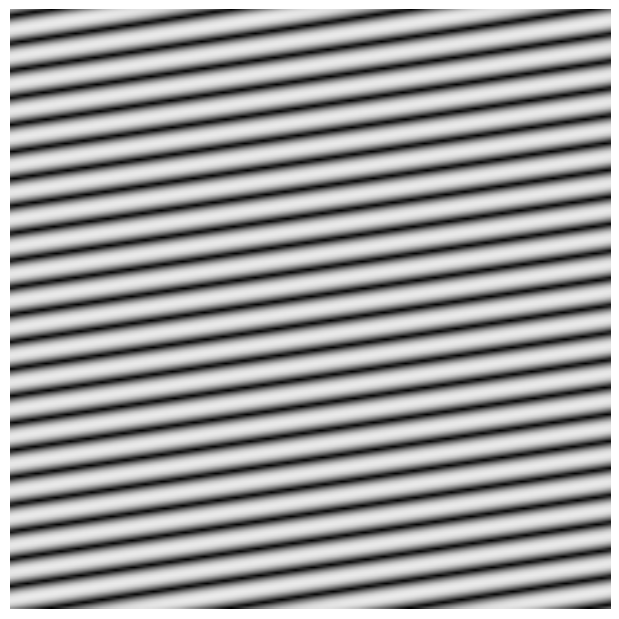} &
        \includegraphics[width=0.15\linewidth]{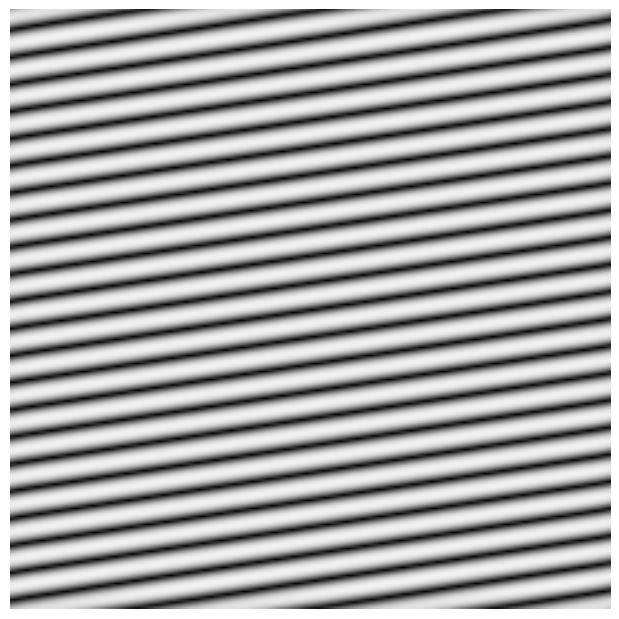} &
        \includegraphics[width=0.15\linewidth]{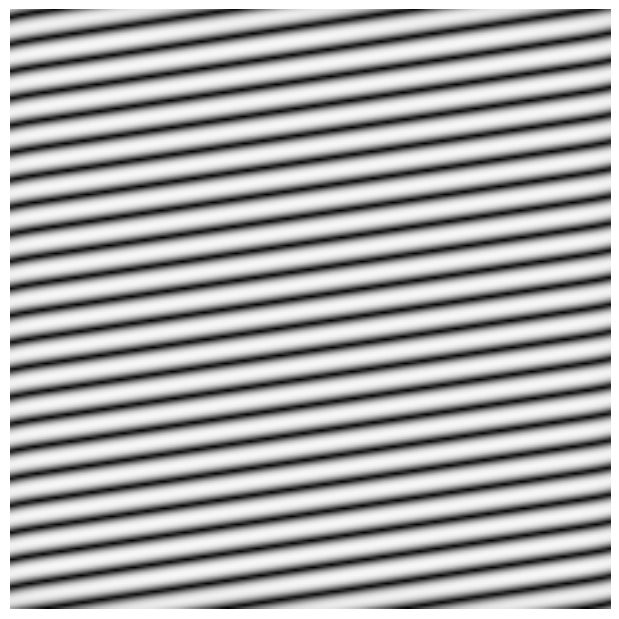} &
        \includegraphics[width=0.15\linewidth]{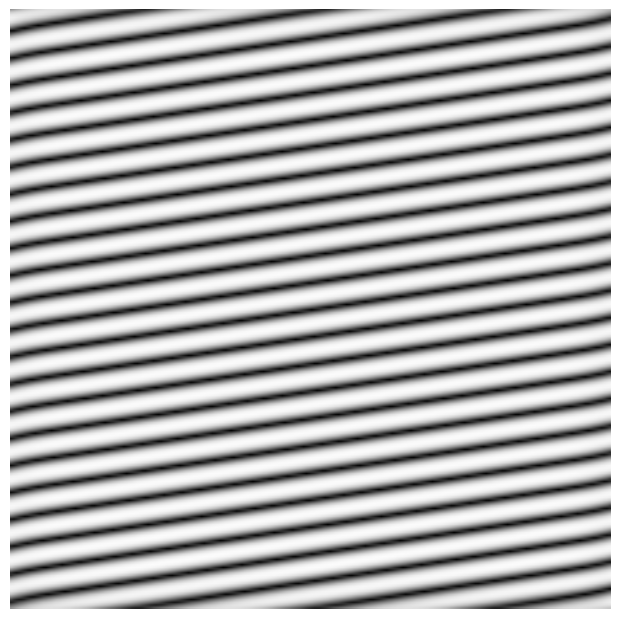}
        \end{tabular}
        \vspace{-1em}
        \subcaption{Responses to striped pattern}\label{fig:stripe_response}
        \vspace{0.5em}
    \end{subfigure}
    
    \begin{subfigure}{\linewidth}
        \centering 
        \tabcolsep=0.1em
        \begin{tabular}{cccccc}
        $m = 1$ & $m = 2$ & $m = 3$ & $m = 4$ & $m = 5$ & $m = 6$ \\
        \includegraphics[width=0.15\linewidth]{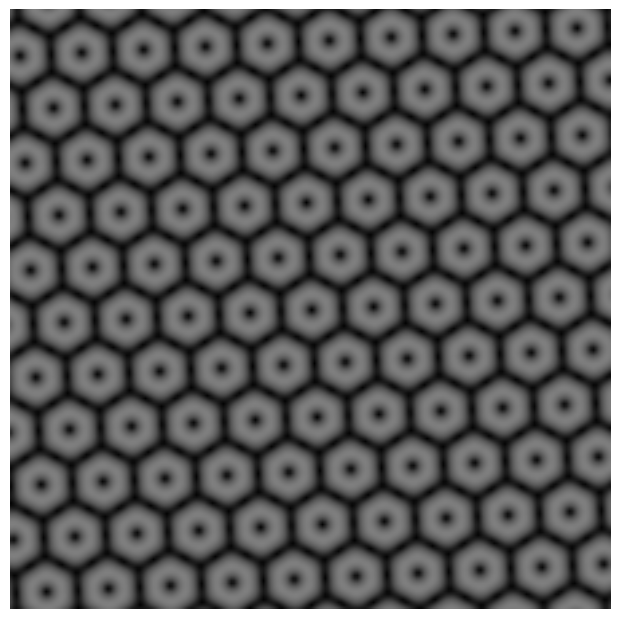} &
        \includegraphics[width=0.15\linewidth]{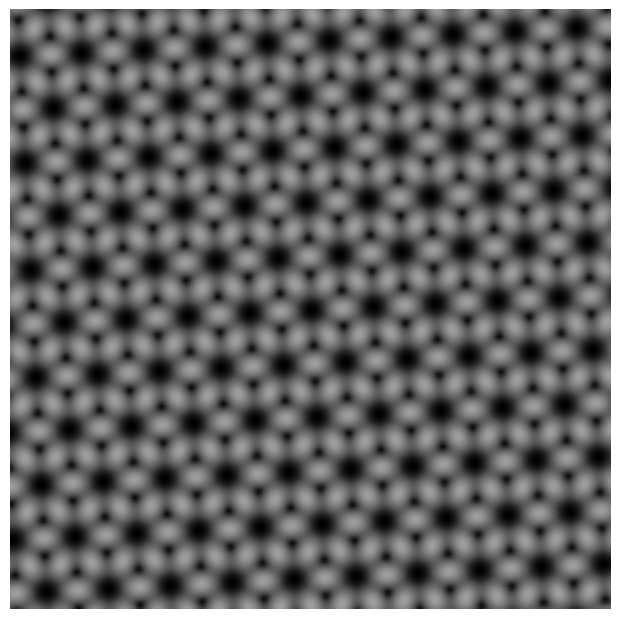} &
        \includegraphics[width=0.15\linewidth]{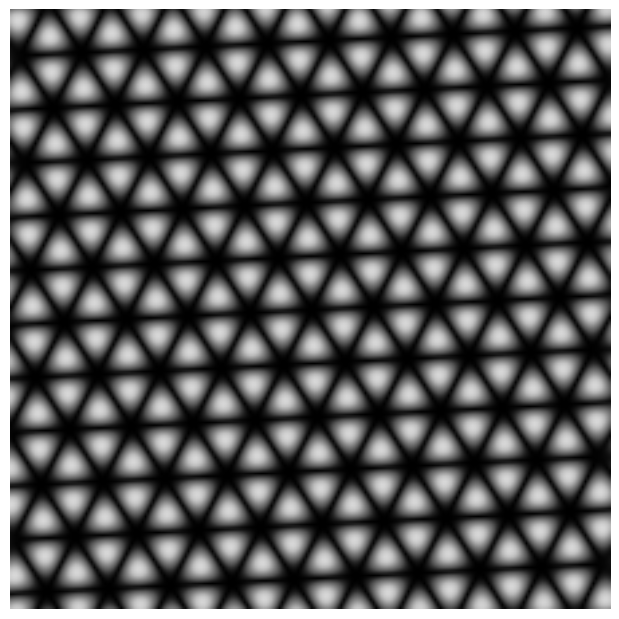} &
        \includegraphics[width=0.15\linewidth]{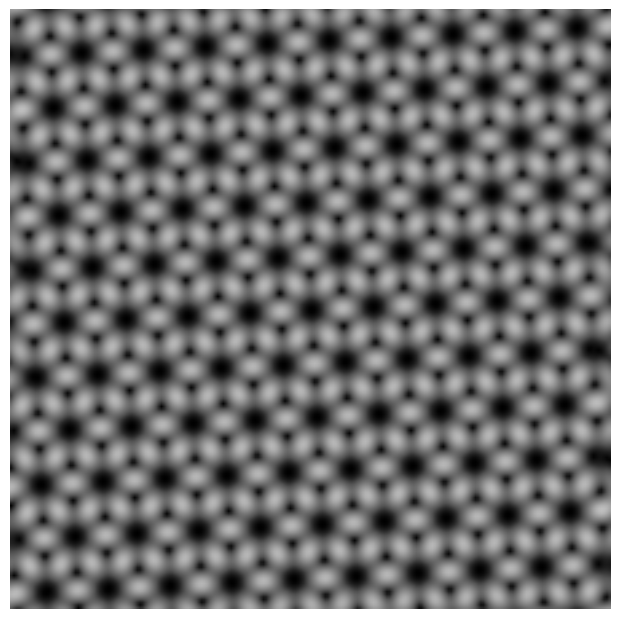} &
        \includegraphics[width=0.15\linewidth]{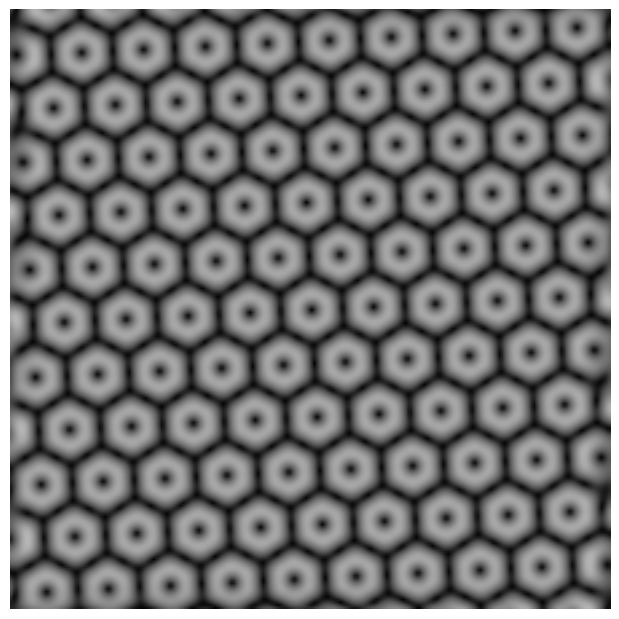} &
        \includegraphics[width=0.15\linewidth]{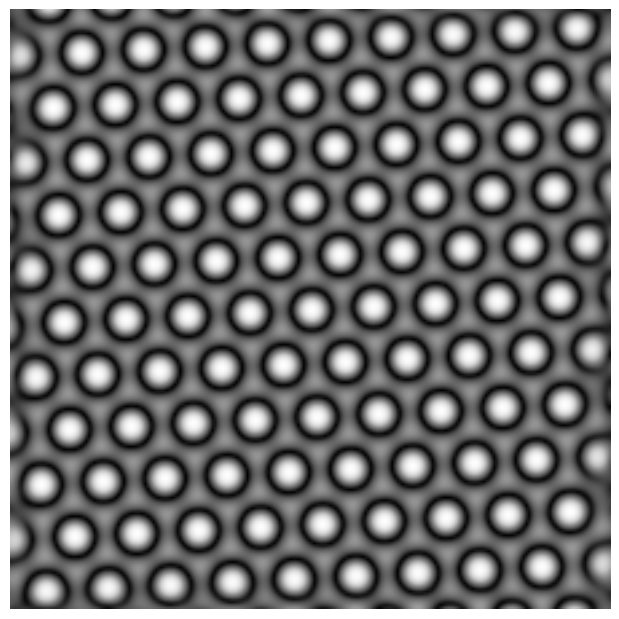}
        \end{tabular}
        \vspace{-1em}
        \subcaption{Responses to hexagonal pattern}\label{fig:hex_response}
        \vspace{0.5em}
    \end{subfigure}
    
   \caption{Responses of steerable polar shapelet filters applied to uniform one-mode approximations of a (a) stripe and (b) hexagonal pattern from eqns. \ref{eq:uniformpattern}-\ref{eqn:uniform_coeffs}.}
  \label{fig:steerable_uniform}
\end{figure}

\subsection{Application to Self-Assembly Imaging}\label{sec:application}

\newcommand{\rv}{\ensuremath{\bm{r}}}
\newcommand{\optwt}{\ensuremath{w^*}}
\newcommand{\nshapes}{\ensuremath{p}}
\newcommand{\refpixels}{\ensuremath{\mathcal{R}}}
\newcommand{\imgpixels}{\ensuremath{\mathcal{I}}}

\begin{figure}
    \centering
    \begin{subfigure}{0.3\linewidth}
        \includegraphics[width=\linewidth]{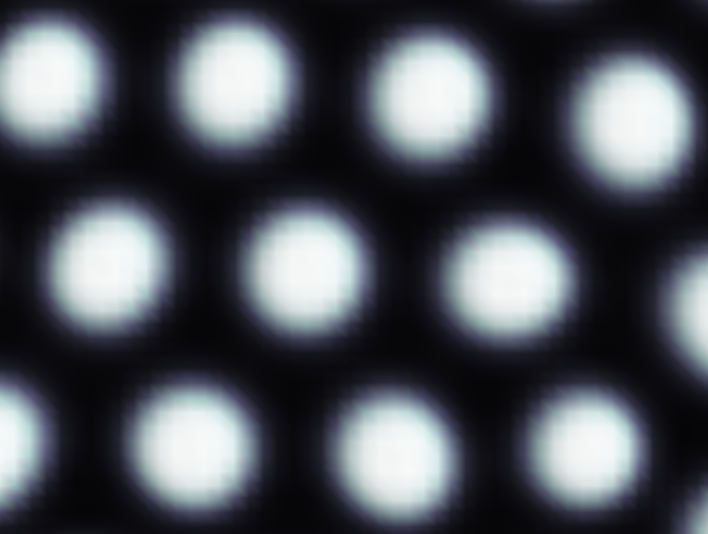}
        \subcaption{}      
    \end{subfigure}
    \begin{subfigure}{0.3\linewidth}
        \includegraphics[width=\linewidth]{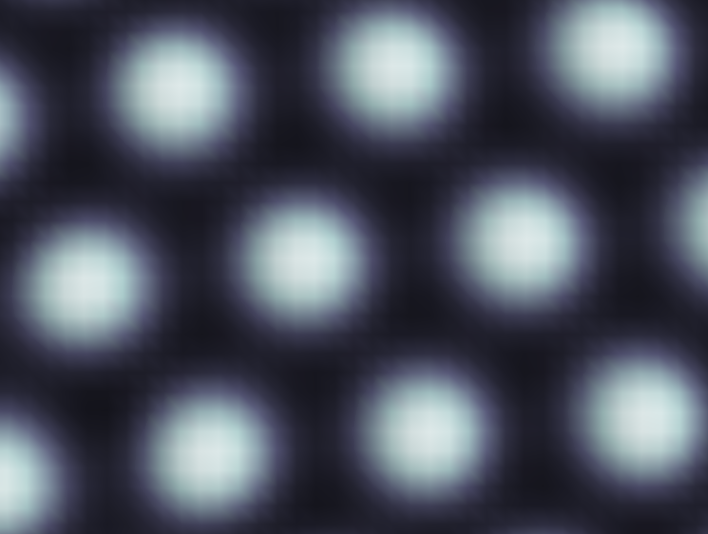}
        \subcaption{}      
    \end{subfigure}
    \begin{subfigure}{0.3\linewidth}
        \includegraphics[width=\linewidth]{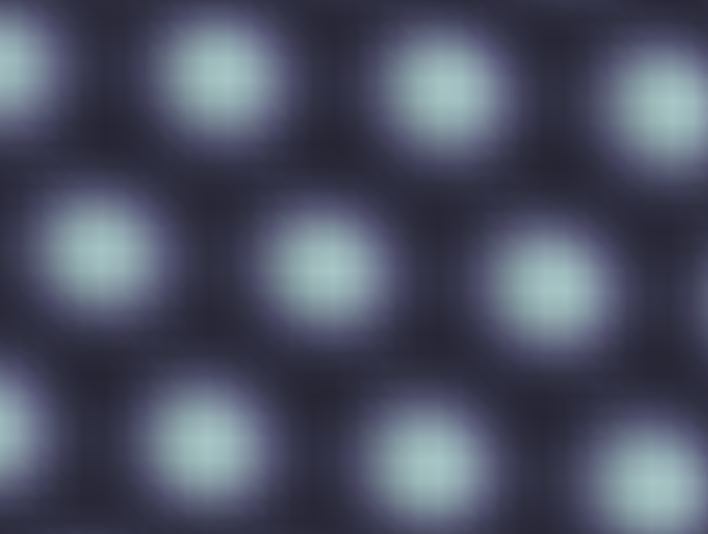}
        \subcaption{}      
    \end{subfigure}
    \caption{Examples of hexagonal surface self-assembly with features of varying character: (a) sharp interface, (b) semi-diffuse interface, (c) diffuse interface. Figures taken from ref. \cite{Abukhdeir2011}}\label{fig:feature_variation}
\end{figure}

Surface self-assembly imaging typically involves surfaces with patterns that are non-uniform and pattern features that are not well-approximated using a one-mode assumption.
Figures \ref{fig:feature_variation} and \ref{fig:test_case} shows example images of two-dimensional surface self-assembly where non-uniform stripe and hexagonal patterns are present (taken from \cite{Abukhdeir2011}).
In Figure \ref{fig:feature_variation} the pattern features themselves vary in shape as is shown in Figure \ref{fig:feature_variation}.
In Figures~\ref{fig:test_stripe}-\ref{fig:test_hex} multiple quasi-uniform subdomains, or ``grains'', are present with defect regions (grain boundaries) at the interface between them.
In order to test the presented shapelet-based method on these realistic patterns, a guided machine learning approach was used to classify regions with uniform patterns from those with defects.

\begin{figure*}[h]
    \begin{subfigure}{0.45\linewidth}
        \begin{tikzpicture}
	        \node (img1){\includegraphics[width=\linewidth]{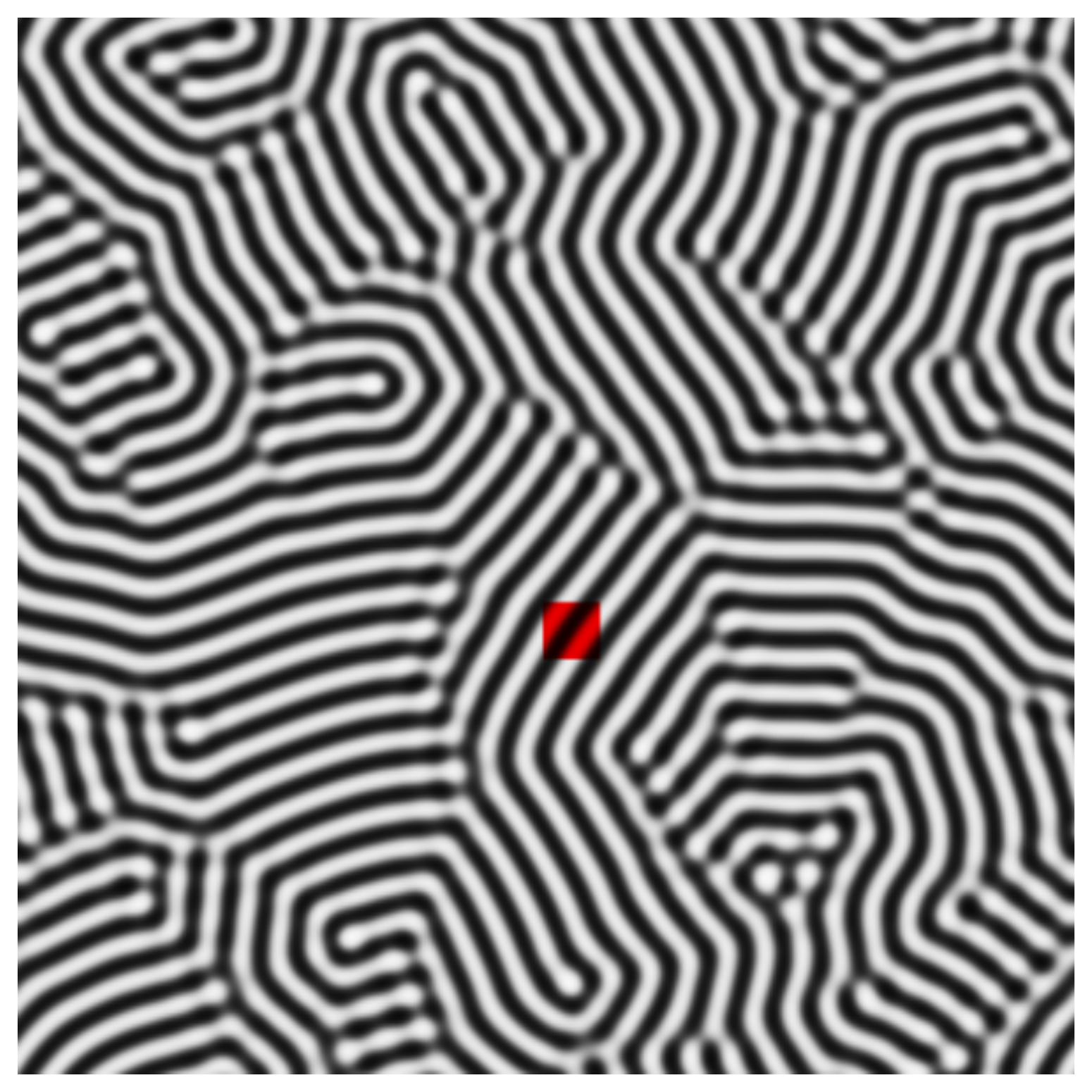}};
	        \node (img2) at (img1.center)[xshift=-0.323\linewidth, yshift=-0.323\linewidth]{\includegraphics[trim = 2in 2in 2in 2in, clip, width=0.33\linewidth]{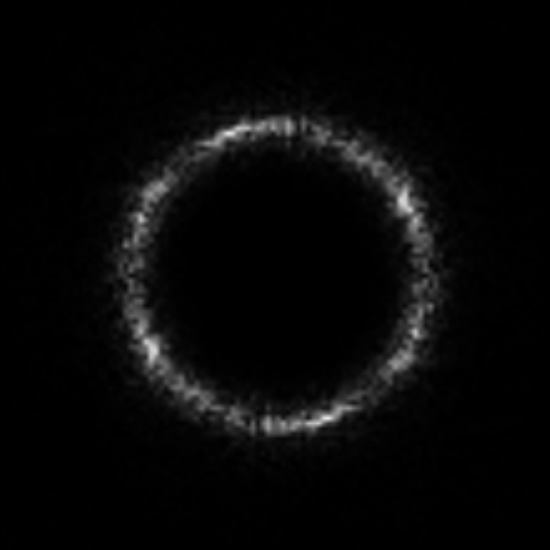}};
        \end{tikzpicture}
        \subcaption{}
        \label{fig:test_stripe}
    \end{subfigure}    
    \hfill  
    \begin{subfigure}{0.45\linewidth}
        \begin{tikzpicture}
	        \node (img1){\includegraphics[width=\linewidth]{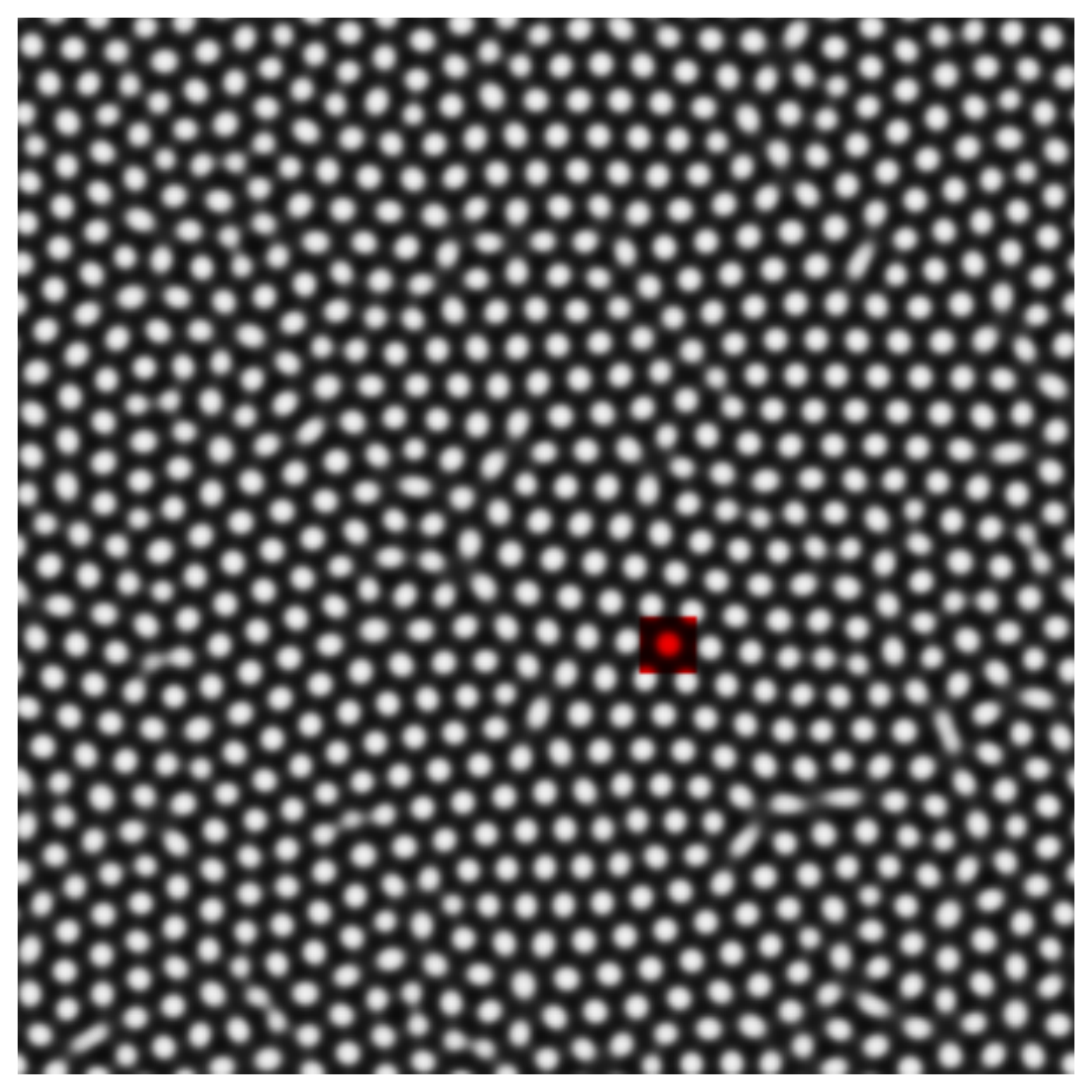}};
	        \node (img2) at (img1.center)[xshift=-0.323\linewidth, yshift=-0.323\linewidth]{\includegraphics[trim = 2in 2in 2in 2in, clip, width=0.33\linewidth]{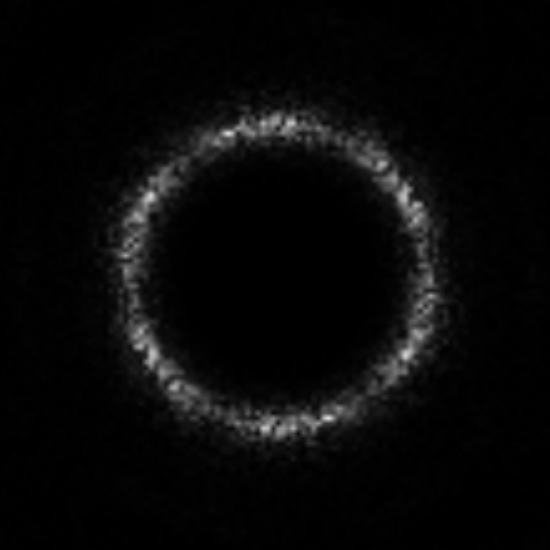}};
        \end{tikzpicture}
        \subcaption{}
        \label{fig:test_hex}
    \end{subfigure}\\      
    \begin{subfigure}{0.45\linewidth}
        \frame{\includegraphics[width=\linewidth]{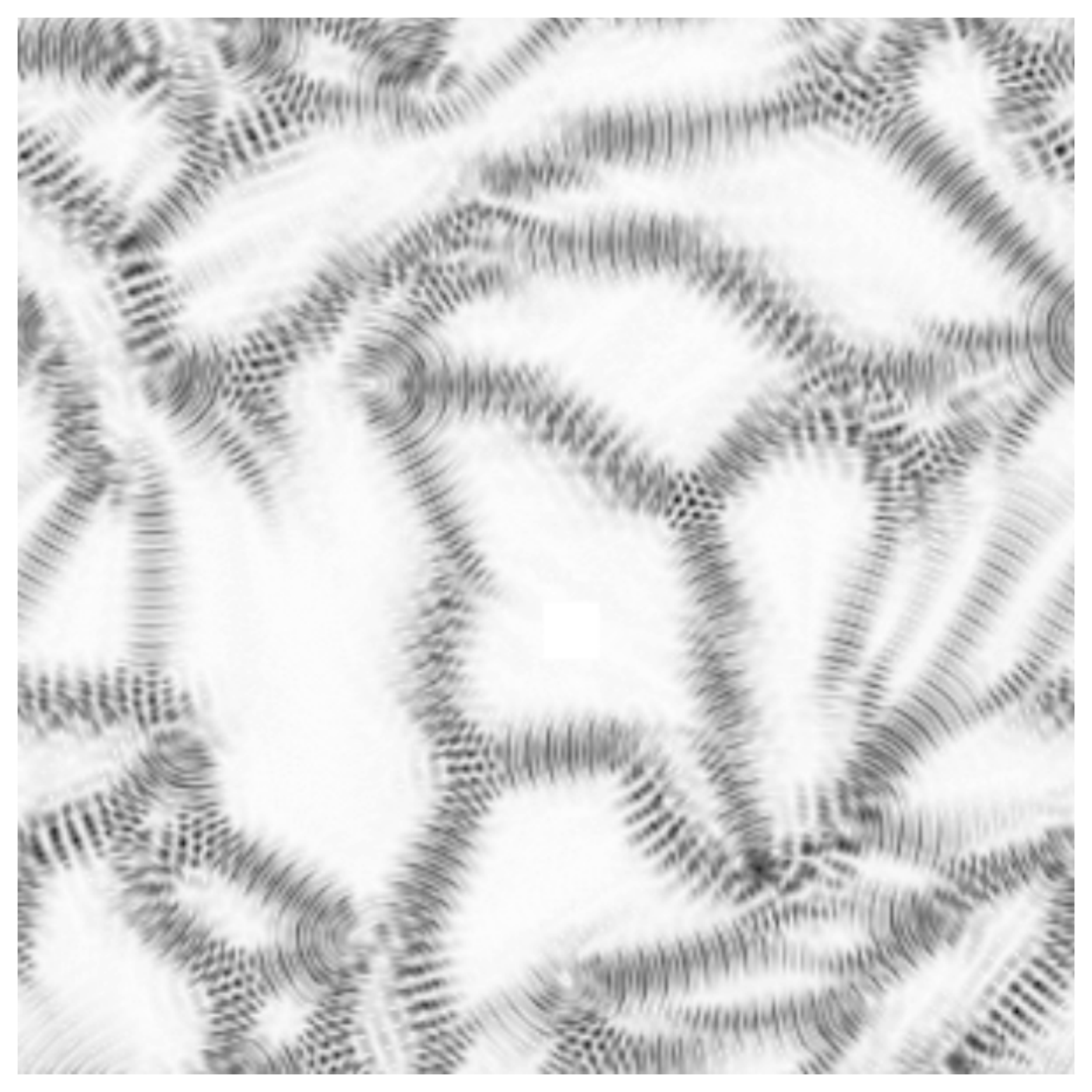}}
        \subcaption{}
        \label{fig:neighbor_stripe}
    \end{subfigure}
    \hfill
    \begin{subfigure}{0.45\linewidth}
        \frame{\includegraphics[width=\linewidth]{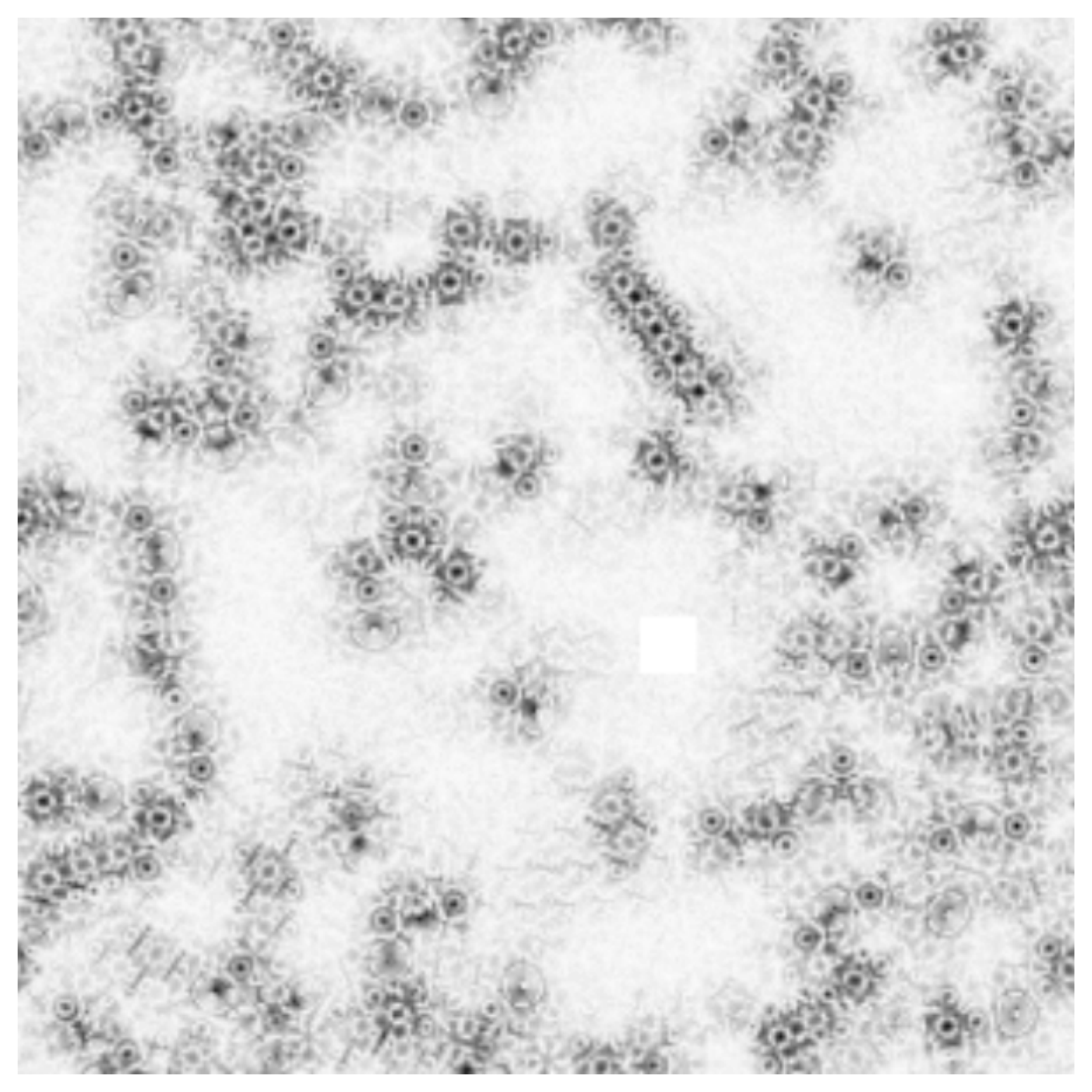}}
        \subcaption{}
        \label{fig:neighbor_hex}
    \end{subfigure}
    \caption{Examples of non-uniform (a) stripe and (b) hexagonal patterns from simulations of surface self-assembly (taken from ref. \cite{Abukhdeir2011}) with inset spectral density plots; Results from applying the guided machine learning algorithm to the (c) stripe and (d) hexagonal patterns where the response distance (eqn. \ref{eqn:response_distance}) was normalized to range from $0$ (black) to $1$ (white) and the user-specified set of pixels are highlighted in (a-b).}\label{fig:test_case}
\end{figure*}

The \emph{response space} is defined as $\rv\in\mathbb{R}^p$ where $p$ is the number of steerable shapelet filters used to quantify the pattern.
Thus at each point in the image $(x,y)$, a response vector $\rv$ is computed,
\begin{equation}
    \rv(x,y) = \frac{[\optwt_{0,1}(x,y), \optwt_{0,2}(x,y),...,\optwt_{0,\nshapes}(x,y)]^\T}{||[\optwt_{0,1}(x,y), \optwt_{0,2}(x,y),...,\optwt_{0,\nshapes}(x,y)]||_2} 
\end{equation}
consisting of the shapelet responses $\optwt_{0,j}(x,y)$ under the optimal orientation for image location $(x,y)$ from eqn.~\ref{eqn:steerable_orientation}.
Given a user-specified set of coordinate pairs (i.e.\ pixel locations) \refpixels{} of a defect-free subdomain of the image, at any location of interest $(x',y')$ in the image the {\em response distance} may be defined from the pixel $(x',y')$ to the reference set.
\begin{equation}\label{eqn:response_distance}
d_{\rv}((x',y'),\refpixels) = \min_{(x,y) \in \refpixels} ||\rv(x',y') - \rv(x,y)||_2 ,
\end{equation}
where $d_{\rv}((x',y'),\refpixels)$ is the Euclidean distance between the response vector at the location of interest and the closest response vector in the reference set. 

The response distance encapsulates how different the image is at location $(x',y')$ from the reference set in terms of the relative shapelet responses.
It serves to highlight areas in the image where defects are present or where no pattern is present.
Such areas have response vectors that have larger $d_{\rv}((x',y'),\refpixels)$ from those where no defects are present.

This application of the steerable shapelets method was applied to the stripe and hexagonal self-assembled domain images shown in Figure \ref{fig:test_case}.
The characteristic pattern wavelength $\lambda$ found through the maximum peak of the spectral density (inset of Figure \ref{fig:test_case}) was used to determine the appropriate shapelet scale factors as described in Section \ref{sec:shapeletselection}.
Figures~\ref{fig:neighbor_stripe}-\ref{fig:neighbor_hex} illustrate the normalized Euclidean distance of the response vectors at each pixel with respect to the response vector of the uniform domain shown in Figure \ref{fig:test_case}.
In this figure, intensity is inversely proportional to $d_{\rv}(\cdot,\refpixels)$ for the given quasi-uniform reference set, which clearly reveals the locations of defects in the image.
Note that response distance is \emph{invariant} to pattern rotations, because the elements of the response vectors are invariant to pattern rotations.

Pattern defects are of two general types: translational and orientational.
These are referred to as dislocation and disclination defects \cite{Brock1992}, respectively, as shown in Figure \ref{fig:test_case}.
In stripe patterns, dislocations correspond to regions where a stripe feature begins ($+$) or ends ($-$).
In hexagonal patterns, dislocations correspond to the beginning ($+$) or end ($-$) of a row of hexagonal features.
Orientational defects, disclinations, are manifested in a rapid transition from one pattern orientation to another.
In stripe patterns, the majority of disclinations involve a $\frac{\pi}{2}$ rotation and, in hexagonal patterns, they involve a $\frac{\pi}{6}$ rotation.

The steerable shapelets method results, shown in Figure \ref{fig:test_case}, show a direct relation between areas of strong response and quasi-uniform areas in the original pattern images.
Areas where response is minimal corresponds to one of three localized cases: (i) the presence of defects, (ii) large strain of the pattern (stripe curvature and or dilation/compression), (iii) deviation of the pattern feature from the one-mode approximation.

With respect to defects present in both images, the image analysis results show good agreement with visual inspection of local topology in the original image.
In areas of large strain of the pattern, which are typically also in the locality of defects, the shapelet response decays smoothly.
This could be considered a drawback in that the method does not strongly distinguish between defect ``core'' regions and the region of strain surrounding the core.
Alternatively, resolving the entirety of the region influenced by a single defect, or cluster of defects, likely has some significant in relating the pattern quality to material properties.
Finally, in both images there are pattern features that strongly deviate from the one-mode approximation of the pattern.
In the stripe pattern there are regions with convex circular shape and in the hexagonal pattern there are regions with lamellar-like features.
The method is both robust in the presence of these features and strongly responds to their presence.

\section{Conclusions}

A method for quantitative analysis of surface self-assembly imaging was presented and applied to images of stripe and hexagonal ordered domains.
A set of orthogonal functions, shapelets, were shown to be useful as filters which respond optimally to surface patterns with n-fold symmetry $n$ is the order of the shapelet.
Steerable formulations of the shapelet functions were derived using steerable filter theory and used to efficiently compute the filter rotation which yields maximal response.
The utility of the steerable shapelet filter approach was demonstrated on uniform stripe and hexagonal patterns.
Furthermore, realistic nonuniform surface patterns were analyzed using the presented steerable shapelet method through guided machine learning.
This approach is able to \emph{quantitatively} distinguish between uniform (defect-free) and non-uniform (strained, defects) regions within the imaged self-assembled domains.
The presented method is both computationally efficient, requiring only two shapelet evaluations per steerable shapelet, and robust in the presence of variation in pattern feature shape.
Finally, the shapelet-based method provides significantly enhanced resolution (pixel-level) compared to the bond-orientational order method (feature-level).

\section*{Acknowledgements}

This work was made possible by the Natural Sciences and Engineering Research Council of Canada (\href{http://www.nserc-crsng.gc.ca/}{NSERC}) and the facilities of the Shared Hierarchical Academic Research Computing Network (\href{www.sharcnet.ca}{SHARCNET}).

\section*{References}

\bibliography{computational,self_assembly,general}

\end{document}